\documentclass[sigconf]{acmart}

\usepackage{graphicx}
\usepackage{booktabs}
\usepackage{amsmath}
\usepackage{microtype}
\usepackage{balance}
\usepackage{multirow}

\usepackage{stfloats}
\usepackage{placeins}
\usepackage{subcaption}
\usepackage{svg}

\usepackage{enumitem}
\usepackage{colortbl}
\usepackage{adjustbox}
\usepackage{cuted}        % provides strip
\usepackage{caption} 
\definecolor{bestcolor}{HTML}{D6EAF8}    % light blue for best
\definecolor{secondcolor}{HTML}{FEF5E7}  % light orange for second best
\definecolor{ourscolor}{HTML}{EBF5FB}    % subtle blue tint for Ours row

\newtheorem{assumption}{Assumption}

\NewDocumentCommand{\zehong}
{ mO{} }{\textcolor{cyan}{\textsuperscript{\textit{Zehong}}\textsf{\textbf{\small[#1]}}}}

\title{Generalizing GNNs with Tokenized Mixture of Experts}

\author{Xiaoguang Guo}
\orcid{0009-0009-9973-8163}
\affiliation{%
  \institution{University of Connecticut}
  \city{Storrs}
  \state{Connecticut}
  \country{USA}}
\email{mca25001@uconn.edu}

\author{Zehong Wang}
\orcid{0000-0002-7670-6777}
\affiliation{%
  \institution{University of Notre Dame}
  \city{Notre Dame}
  \state{Indiana}
  \country{USA}}
\email{zwang43@nd.edu}

\author{Jiazheng Li}
\orcid{0009-0000-0672-389X}
\affiliation{%
  \institution{University of Connecticut}
  \city{Storrs}
  \state{Connecticut}
  \country{USA}}
\email{jiazheng.li@uconn.edu}

\author{Shawn Spitzel}
\orcid{0009-0009-5949-7000} 
\affiliation{%
  \institution{University of Connecticut}
  \city{Storrs}
  \state{Connecticut}
  \country{USA}}
\email{shawn.spitzel@uconn.edu}

\author{Qi Yang}
\orcid{0009-0000-1808-0537}  
\authornote{Visiting Student at the University of Connecticut.}
\affiliation{%
  \institution{University of Connecticut}
  \city{Storrs}
  \state{Connecticut}
  \country{USA}}
\email{qiyang942@gmail.com}

\author{Kaize Ding}
\orcid{0000-0001-6684-6752}
\affiliation{%
  \institution{Northwestern University}
  \city{Evanston}
  \state{Illinois}
  \country{USA}}
\email{kaize.ding@northwestern.edu}

\author{Jundong Li}
\orcid{0000-0002-1878-817X}
\affiliation{%
  \institution{University of Virginia}
  \city{Charlottesville}
  \state{Virginia}
  \country{USA}}
\email{jundong@virginia.edu}

\author{Chuxu Zhang}
\orcid{0000-0002-8349-7926}
\authornote{Corresponding author.}
\affiliation{%
  \institution{University of Connecticut}
  \city{Storrs}
  \state{Connecticut}
  \country{USA}}
\email{chuxu.zhang@uconn.edu}
\copyrightyear{2026}
\acmYear{2026}
\setcopyright{cc}
\setcctype{by}
\acmConference[KDD '26]{Proceedings of the 32nd ACM SIGKDD Conference on Knowledge Discovery and Data Mining V.2}{August 09--13, 2026}{Jeju Island, Republic of Korea}
\acmBooktitle{Proceedings of the 32nd ACM SIGKDD Conference on Knowledge Discovery and Data Mining V.2 (KDD '26), August 09--13, 2026, Jeju Island, Republic of Korea}
\acmDOI{10.1145/3770855.3817952}
\acmISBN{979-8-4007-2259-2/2026/08}

\begin{document}

\begin{abstract}
Deployed graph neural networks (GNNs) operate as frozen snapshots, yet must simultaneously fit clean data, generalize under distribution shifts, and remain stable against input perturbations---three goals that are difficult to satisfy at once with a single fixed model.
We first show theoretically that a single fixed inference rule can create a stability--generalization tradeoff: making the model insensitive to perturbations can also suppress task-relevant signals needed to fit and generalize. Input-dependent routing---assigning different computation paths to different inputs---can relax this tension, but brings new fragility: distribution shifts may mislead routing decisions, and perturbations can destabilize routing, compounding downstream errors.
We formalize these effects through two risk decompositions that separate (i)~how well the available paths cover diverse test conditions from how accurately the router selects among them, and (ii)~how sensitive each fixed path is from how much routing fluctuation amplifies that sensitivity.
Guided by these analyses, we propose \textbf{STEM-GNN}: \textbf{S}table \textbf{T}ok\textbf{E}nized \textbf{M}ixture-of-Experts \textbf{GNN}, a pretrain-then-finetune framework that couples a \emph{mixture-of-experts encoder} providing diverse computation paths to cover heterogeneous test conditions, a \emph{vector-quantized token interface} that maps encoder outputs to a discrete codebook to absorb small representation drift induced by input perturbations and routing fluctuations before downstream layers, and a \emph{Lipschitz-regularized prediction head} that bounds how much the output can amplify residual upstream variation. Across eight node, link, and graph benchmarks, STEM-GNN maintains strong clean performance; on representative node benchmarks, it improves the three-way balance under degree/homophily shifts and feature/edge perturbations. The code and data are available at \url{https://github.com/GXG-CS/STEM-GNN}.

\end{abstract}

%% CCS concepts (generated via ACM CCS tool: https://dl.acm.org/ccs)


% \keywords{}

\maketitle

\section{Introduction}

Graph neural networks (GNNs) have been applied to many real-world systems, powering applications from recommendation and information retrieval to molecular property prediction and knowledge reasoning~\cite{battaglia2018relational,bronstein2021geometric,hamilton2017inductive,gilmer2017neural,sanchez2020learning, ju2022grape, DBLP:journals/corr/abs-2502-12908}.
In production pipelines, GNNs are typically deployed as versioned snapshots that run without parameter updates between releases~\cite{ying2018graph,sculley2015hidden,borisyuk2024lignn,hu2020open}.
During this \emph{frozen deployment} window, the model encounters two distinct sources of mismatch.
\emph{Distribution shifts} arise when the test graph is drawn from a different regime than the training data~\cite{gui2022good}---for example, a molecular model trained on common scaffolds may fail on novel chemical spaces~\cite{wu2018moleculenet}.
\emph{Feature and structure perturbations} arise when individual test inputs are corrupted by missing attributes, noisy interactions, or mild edge rewiring while remaining semantically close to their clean counterparts~\cite{zugner2018adversarial,jin2020graph,wang2024safety}.
Because the model cannot adapt at test time, strong clean performance, reliable behavior under shifts, and stable behavior under perturbations must all be achieved by a single set of frozen parameters---a tri-objective tension we term the \emph{impossible triangle} of frozen graph deployment.

\begin{figure}[t]
  \centering
  \includegraphics[width=\linewidth]{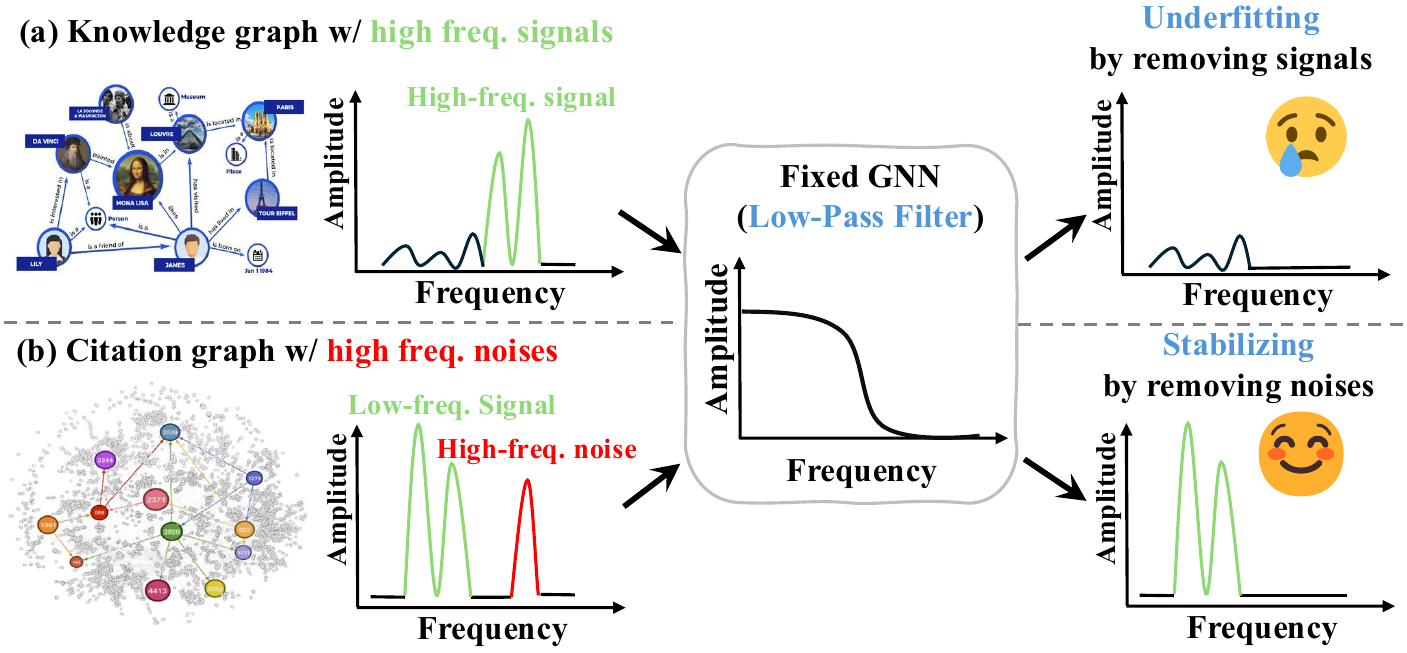}
  \vspace{-0.25in}
  \caption{A single fixed low-pass GNN cannot serve both regimes: on a knowledge graph where high-frequency components carry task-relevant signal, low-pass filtering discards it (underfitting); on a citation graph where high-frequency components are noise, the same filtering happens to suppress it (stabilizing). No fixed rule fits both, motivating input-dependent computation.}
    \vspace{-0.25in}
\label{fig:motivation_fixed_low_pass}
\end{figure}

Prior work has advanced each axis of this triangle largely in isolation.
Graph pretraining~\cite{velivckovic2018deep,hou2022graphmae,chien2021node,wang2024gft, zhao2021multi, wang2025towards, wang2024subgraph, thapaliya2025semantic} learns transferable representations that improve clean accuracy and downstream adaptation, yet the resulting model still applies a single fixed forward pass at test time.
OOD graph learning~\cite{li2025out,wu2024graph,li2022learning,wu2024graphmetro,zhu2024mario, wang2025generative} targets distribution shifts across environments but often treats perturbation robustness as secondary.
Robust graph learning~\cite{zhang2020gnnguard,rong2019dropedge,zugner2018adversarial,zhang2022chasing,yuan2024mitigating,yuan2024graph} hardens predictors against corruptions but can sacrifice expressivity needed for generalization.
Crucially, most methods deploy a \emph{fixed test-time computation rule}: the same message-passing and readout mapping is executed for every input regardless of its characteristics.
This raises a fundamental question: \emph{how does the choice of test-time computation mechanism---not just the training objective---govern the tri-objective balance under frozen deployment?}

We adopt the concept of \emph{instance-conditional computation} (ICC) to answer the question: routing different inputs through different mechanisms within a single frozen model~\cite{shazeer2017outrageously,graves2016adaptive,bengio2015conditional}.
ICC expands the family of effective mechanisms that a frozen parameter set can realize, improving coverage over heterogeneous deployment scenarios where a single fixed rule falls short (Figure~\ref{fig:motivation_fixed_low_pass}).
However, ICC also introduces new failure modes.
Under distribution shifts, routing weights can drift, assigning inputs to suboptimal mechanisms.
Under perturbations, instability manifests in two complementary ways: (i)~even when the routing weights are unchanged, the effective mechanism may still be sensitive to input noise; (ii)~perturbations may shift the routing weights, altering the effective mechanism and compounding the output variation.
Therefore, making ICC robust under frozen deployment requires coupling expanded mechanism coverage with explicit deployment-time stability control---a design goal that existing ICC and MoE methods do not jointly address in graph learning.

We tackle this challenge from both a theoretical and a methodological perspective.
On the \textbf{theory} side, we formalize static inference (fixed rule) versus ICC under the tri-objective framework and show that: (i)~static inference faces a structural tension---Lipschitz stability control caps the mechanism's reliance on perturbation-sensitive components, creating a floor on worst-environment risk; and (ii)~ICC decomposes OOD risk into mechanism coverage and selection quality, and stability risk into routing-fixed sensitivity and routing drift, revealing precise levers absent in the static case.

On the \textbf{methodology} side, we propose \textbf{STEM-GNN} (\textbf{S}table \textbf{T}ok\textbf{E}nized \textbf{M}ixture-of-Experts \textbf{GNN}), a pretrain-then-finetune framework that instantiates robust ICC for graph learning through three tightly coupled designs:
\textit{(1) MoE Encoder for Coverage Expansion.}
A mixture-of-experts message-passing encoder routes each node through an input-dependent combination of shared experts, realizing a family of effective mechanisms under a single frozen parameter set.
\textit{(2) VQ Tokenization for Representation Stabilization.}
A vector-quantized (VQ) token interface discretizes encoder outputs into a fixed codebook before passing them to the prediction head. Continuous drift from input perturbations or routing fluctuations that does not cross quantization boundaries produces zero change in the head's input, stabilizing the encoder-to-head pathway.
\textit{(3) Lipschitz Regularization for Sensitivity Control.}
A Frobenius penalty on the prediction head bounds its Lipschitz constant, limiting how strongly any residual change---including discrete token switches---is amplified into output variation.
We summarize our contributions as follows:
\begin{itemize}
[leftmargin=*]\setlength{\itemsep}{0pt}
  \item \textbf{Concept.} We reframe robust graph generalization under frozen deployment as a tri-objective tension and identify the test-time computation mechanism---not the training objective alone---as the key lever governing this balance.
  \item \textbf{Theory.} We formalize static inference vs. ICC under this framework: constructing a spectral witness family, we prove that stability control caps static reliance and induces a worst-environment floor; we further show that ICC offers new design levers via coverage–selection and sensitivity–drift decompositions.
  \item \textbf{Methodology.} Based on theoretical analysis, we propose STEM-GNN, a pretrain-then-finetune framework that couples an MoE encoder, VQ tokenization, and Lipschitz-controlled prediction to operationalize robust ICC for graphs.
  \item \textbf{Experiments.} Across eight benchmarks spanning node, link, and graph tasks, STEM-GNN maintains strong clean accuracy; on representative node benchmarks, it achieves the best tri-objective balance under OOD shifts and perturbations.
\end{itemize}

\vspace{-0.1in}

\section{Related Work}
\textbf{Graph Neural Networks.}
GNNs learn functions over graph structured data primarily through neighborhood aggregation and message passing, enabling topology- and attribute-aware representation learning~\cite{scarselli2008graph, bronstein2017geometric, shuman2013emerging, bruna2013spectral}.
Prominent backbones include convolutional approximations~\cite{kipf2016semi, defferrard2016convolutional}, attentional aggregation~\cite{velickovic2017graph, zhang2018gaan,zhang2019heterogeneous}, scalable inductive frameworks~\cite{hamilton2017inductive, chen2018fastgcn}, deep residual networks~\cite{li2019deepgcns, chen2020simple}, and graph transformers for global dependency modeling~\cite{dwivedi2020generalization, ying2021transformers}.
Despite these architectural variations, standard backbones typically execute a single fixed computation rule under frozen deployment, locking in their test-time inductive bias.
Under heterogeneous conditions (degree, homophily, or perturbations), this fixed-rule rigidity creates a tri-objective tension among fit, OOD generalization, and perturbation stability.

\vspace{3pt}
\noindent\textbf{Robust Graph Generalization.}
Research on robust graph generalization has progressed along three intertwined threads: transferable pretraining, OOD generalization, and perturbation robustness, yet these lines are often developed with different objectives and evaluation protocols.
Graph pretraining learns transferable representations through contrastive paradigms that encourage smoothness~\cite{velivckovic2018deep, you2020graph, zhu2020deep, qiu2020gcc, qian2022co, wen2024coarse, wen2024gcvr, li2026graph} or generative objectives that better preserve local structural details~\cite{hou2022graphmae, chien2021node, tian2023heterogeneous, wang2024gft, wang2025generative, wang2025beyond}.
A complementary line tokenizes structural, higher-order, temporal, or molecular patterns into compositional tokens for transferable representation~\cite{wang2025generative, wang2025beyond, zhao2026hypergraph, ma2026temporal, wang2026molecular}.
However, downstream inference after pretraining still follows a fixed test-time rule, limiting adaptability under shift.
OOD generalization methods address such shifts through invariant learning with test-time mechanisms~\cite{li2022learning, wu2022discovering, zhu2024mario, wang2025generative_tmlr} or adaptive architectures with instance-conditional routing~\cite{wu2024graph, wu2024graphmetro, liu2023fair}.
While instance-conditional routing expands mechanism coverage, existing designs often lack explicit deployment-time stability control (e.g., representation stabilization or sensitivity regularization).
In parallel, perturbation-robust GNNs harden predictors through denoising, pruning, or smoothing under a fixed computation rule~\cite{zhang2020gnnguard, jin2020graph, rong2019dropedge, zhao2023self}, yet these static mechanisms can sacrifice expressivity needed under distribution shifts.
Across these threads, methods typically optimize one or two objectives while treating the rest as auxiliary metrics.
We propose STEM-GNN, which unifies instance-conditional computation with deployment-time stability control for robust graph generalization under frozen deployment.

\section{Problem Formulation and Analysis}
\label{sec:theory}

We formulate and analyze robust graph learning for deployed GNNs under frozen
deployment through a tri-objective lens: \textbf{fit} (clean performance),
\textbf{ood} (worst-environment generalization), and \textbf{stab} (perturbation
robustness).
The central question is how the \emph{inference rule under frozen parameters}
shapes these tradeoffs; we contrast static inference ($\mathcal{H}_1$) and
instance-conditional computation ($\mathcal{H}_2$).

\subsection{Problem Formulation}
\label{sec:theory_probdef}

We study a graph neural network (GNN) $f_\theta:\mathcal{Z}\to\mathcal{U}$ under
frozen deployment (no test-time updates).
Input $z=(G,X)$ consists of an undirected graph $G=(V,E)$ and node features
$X\in\mathbb{R}^{|V|\times d}$; the GNN computes output $u=f_\theta(z)$ via
message passing.
Training uses distribution $D_0$, while test environments $E_{\mathrm{test}}$
induce distributions $\{D_e\}$.
Our deployment objective is captured by three risks: fitting on clean graphs,
worst-environment OOD generalization, and inference-time stability.
To model inference-time perturbations, we specify feature/structure discrepancy measures
$\Delta_x,\Delta_s$ together with budgets $(\rho_x,\rho_s)$.
For input $z=(G,X)$, the admissible perturbation set is
\begin{equation}
B(z) := \{\, z'=(G',X')\ |\ \Delta_x(X',X)\le \rho_x,\
\Delta_s(G',G)\le \rho_s \}.
\label{eq:prelim_Bz}
\end{equation}
For the static-inference witness analysis (\S\ref{sec:H1_main_th}), we specialize Eq.~\eqref{eq:prelim_Bz} to feature perturbations ($\rho_s{=}0$) and write $\rho$ for the induced bound on the high-component perturbation; the general set $B(z)$ retains both feature and structure budgets.
Let $\ell:\mathcal{U}\times\mathcal{Y}\to\mathbb{R}_+$ be the task loss
and $d_{\mathrm{out}}:\mathcal{U}\times\mathcal{U}\to\mathbb{R}_+$ an output discrepancy.
We formalize the three risks as:
\begin{align}
\mathcal{R}_{\mathrm{fit}}(\theta) &:= \mathbb{E}_{(z,y)\sim D_0}[\ell(f_\theta(z),y)],
\label{eq:prelim_Rfit}
\\
\mathcal{R}_{\mathrm{ood}}(\theta) &:= \sup_{e\in E_{\mathrm{test}}}\
\mathbb{E}_{(z,y)\sim D_e}[\ell(f_\theta(z),y)],
\label{eq:prelim_Rood}
\\
\mathcal{R}_{\mathrm{stab}}(\theta) &:= \mathbb{E}_{(z,y)\sim D_0}
\Big[\sup_{z'\in B(z)} d_{\mathrm{out}}(f_\theta(z), f_\theta(z'))\Big].
\label{eq:prelim_Rstab}
\end{align}

\noindent\textbf{Analysis roadmap.}
Since the predictor is frozen, the tradeoff among these risks is shaped by the
inference rule under fixed parameters.
Hence, we contrast two paradigms:
\textbf{(i)~Static inference} ($\mathcal{H}_1$) applies a single message-passing rule to all inputs, yielding a stability-generalization tension (Sec.~\ref{sec:H1_main_th});
\textbf{(ii)~Instance-conditional computation} ($\mathcal{H}_2$) routes inputs through different computation paths, relaxing the $\mathcal{H}_1$ limitation but introducing routing instability under shifts and perturbations (Sec.~\ref{sec:H2_main_th}).
These insights inform STEM-GNN (Sec.~\ref{sec:method}).

\subsection{$\mathcal{H}_1$: A Witness Tension for Static Inference}
\label{sec:H1_main_th}

Static inference deploys a \emph{single frozen GNN computation rule} for all inputs: the same message-passing layers and readout are applied to every graph and node under fixed parameters.
For GNNs, this creates a fundamental tension.
To remain stable to feature/structure perturbations, the model must limit reliance on perturbation sensitive components (often associated with high-frequency, non-smooth signals on the graph), yet these same components can be crucial for fitting and worst-environment generalization.
To isolate this GNN-specific tradeoff, we analyze a minimal witness family where a single scalar $\eta_\theta$ controls such reliance.

\vspace{3pt}\noindent\textit{Static inference and a minimal witness family.}
Consider the witness family~\cite{shalev2014understanding}:
\begin{equation}
f_\theta(G,X)\ :=\ h_\theta\!\big(X_{\mathrm{L}} + \eta_\theta X_{\mathrm{H}}\big),
\qquad \eta_\theta\ge 0,
\label{eq:H1_proxy_mix_main_th}
\end{equation}
where $X_{\mathrm{L}}$ and $X_{\mathrm{H}}$ represent complementary \emph{graph-induced spectral components} of node features, e.g.,
$X_{\mathrm{L}}=P_{\mathrm{L}}(G)X$ and $X_{\mathrm{H}}=P_{\mathrm{H}}(G)X$ for graph filters $P_{\mathrm{L}},P_{\mathrm{H}}$ that separate low-pass (smooth) and high-pass (non-smooth) information~\cite{shuman2013emerging, defferrard2016convolutional}.
The scalar reliance $\eta_\theta$ is fixed after training and shared across environments at deployment. Because $X_{\mathrm{L}}$ and $X_{\mathrm{H}}$ are induced by graph filters $P_{\mathrm{L}}(G)$ and $P_{\mathrm{H}}(G)$, the resulting witness tension is tied to graph-dependent message passing rather than an input-independent feature split.

For prescribed $\alpha>0$ and $\varepsilon\ge 0$, and given the Lipschitz stability constraint $L_h\,\eta_\theta\,\rho\le\varepsilon$, we define the slice-optimal worst-environment risk for the witness subfamily as:
\begin{equation}
\beta_{1}(\alpha,\varepsilon)
:=
\inf_{\theta\in\mathcal{H}_1^{\mathrm{wit}}:\ \mathcal{R}_{\mathrm{fit}}(\theta)\le \alpha,\ L_h\,\eta_\theta\,\rho\le \varepsilon}
\ \mathcal{R}_{\mathrm{ood}}(\theta).
\label{eq:beta1_H1_def_th}
\end{equation}
This worst-environment objective is standard in robust risk formulations under
distribution shift~\cite{sagawa2019distributionally}. In STEM-GNN, this constraint is encouraged through the Frobenius penalty on the prediction head (Eq.~\eqref{eq:lip_fro_penalty}), which softly controls the head Lipschitz constant during training.

\vspace{3pt}\noindent\textbf{Assumptions.}
Admissible perturbations mainly act on the high component~\cite{zugner2018adversarial,guo2026safety},
with magnitude at most $\rho$.
The readout $h_\theta$ is $L_h$-Lipschitz~\cite{miyato2018spectral} from the
GNN embedding norm to the output metric.
Good fitting (and at least one test environment) requires nontrivial reliance
on the high component.
Complete formal definitions and the proof are deferred to Appendix~\ref{app:H1}.

We now show that a Lipschitz stability constraint induces an unavoidable cap on
$\eta_\theta$, which in turn bounds worst-environment risk away from zero.

\begin{theorem}[Witness-family tension under a Lipschitz stability constraint]
\label{thm:H1_lower_main_th}
Under the assumptions, define:
\begin{equation}
\eta_{\max}(\varepsilon)\ :=\ \frac{\varepsilon}{L_h\,\rho},
\qquad
\eta_{\min}(\alpha)\ :=\ \inf\{\eta\ge 0:\ \psi_{\mathrm{fit}}(\eta)\le \alpha\},
\label{eq:H1_eta_minmax_main_th}
\end{equation}
where $\psi_{\mathrm{fit}}(\eta)$ is the monotone lower-bound function relating fit risk to reliance, and $e_1\in E_{\mathrm{test}}$ is the witness environment with associated monotone lower bound $\psi_{e_1}(\eta)$.
If $\eta_{\min}(\alpha)>\eta_{\max}(\varepsilon)$, then no GNN in
\eqref{eq:H1_proxy_mix_main_th} can satisfy both
$\mathcal{R}_{\mathrm{fit}}(\theta)\le \alpha$ and the constraint $L_h\,\eta_\theta\,\rho\le\varepsilon$.
Otherwise, for any $\theta$ in \eqref{eq:H1_proxy_mix_main_th} satisfying both
$\mathcal{R}_{\mathrm{fit}}(\theta)\le \alpha$ and $L_h\,\eta_\theta\,\rho\le\varepsilon$, it holds that
\begin{equation}
\mathcal{R}_{\mathrm{ood}}(\theta)
\ \ge\
\psi_{e_1}\!\big(\eta_{\max}(\varepsilon)\big)
\ =:\ \underline{\beta}_1(\alpha,\varepsilon)
\ >\ 0.
\label{eq:H1_ood_lower_main_th}
\end{equation}
\end{theorem}

Consequently, whenever the feasible slice is nonempty, the slice-optimal risk satisfies $\beta_1(\alpha,\varepsilon)\ge\underline{\beta}_1(\alpha,\varepsilon)>0$, closing the loop between the slice-optimal objective in Eq.~\eqref{eq:beta1_H1_def_th} and the witness lower bound.

This result characterizes a witness-family tension rather than a universal impossibility result for all static GNN parameterizations; extending the analysis beyond this family is left for future work. The Lipschitz stability constraint caps reliance at $\eta_\theta\le\eta_{\max}(\varepsilon)$ because larger reliance on $X_{\mathrm{H}}$ amplifies perturbations through the $L_h$-Lipschitz readout, while achieving fit risk at most $\alpha$ requires $\eta_\theta\ge\eta_{\min}(\alpha)$ since the high component carries task-relevant signal. Thus, under a Lipschitz stability constraint, static GNN inference within this witness subfamily faces a fit--stability--OOD tension that motivates instance-conditional computation.

% =========================
% H2: ICC levers
% =========================
\subsection{$\mathcal{H}_2$: ICC Exposes Coverage and Stability Levers}
\label{sec:H2_main_th}

Section~\ref{sec:H1_main_th} showed that a single frozen rule must share one global reliance pattern, producing a positive worst-environment floor under a Lipschitz stability constraint. Instance-conditional computation (ICC) relaxes this rigidity by routing inputs to different \emph{effective message-passing mechanisms}, expanding the mechanism family available at deployment; the cost is a new stability pathway, since routing decisions can drift under distribution shifts or admissible perturbations.

\smallskip\noindent\textbf{ICC class.}
We define the ICC class~\cite{bengio2015conditional} of a GNN $f_\theta$ under frozen deployment: there exist a routing space $\mathcal{R}$, a routing map $r_\theta:\mathcal{Z}\to\mathcal{R}$, and an execution map $F:\mathcal{R}\times\mathcal{Z}\to\mathcal{U}$ such that
\begin{equation}
f_\theta(z)\ :=\ F\big(r_\theta(z),\, z\big),\qquad z=(G,X),
\label{eq:H2_general_form_main_th}
\end{equation}
where $r_\theta$ selects an input-dependent routed state and $F$ executes the corresponding routed GNN computation~\cite{rosenbaum2017routing}. This representation lets us isolate two ICC-specific levers.

\smallskip\noindent\textbf{Lever 1: coverage and selection.}
Consider an ICC instantiation that selects from a finite family of mechanisms~\cite{jacobs1991adaptive}. Let $\beta_{\mathrm{cov}}$ denote worst-environment coverage (the best loss achievable in the family per environment), $\delta_{\mathrm{sel}}(e)$ the selection error in environment $e$, and assume the evaluation loss is bounded by $L_{\max}$ on $E_{\mathrm{test}}$ (used only for analysis). Then
\begin{equation}
\mathcal{R}_{\mathrm{ood}}(\theta)
\ \le\
\beta_{\mathrm{cov}}
+
L_{\max}\cdot \sup_{e\in E_{\mathrm{test}}}\delta_{\mathrm{sel}}(e),
\label{eq:H2_ood_bound_main_th}
\end{equation}
isolating two design knobs: expand the mechanism family (reduce $\beta_{\mathrm{cov}}$) and stabilize routing under shift (reduce $\sup_e\delta_{\mathrm{sel}}(e)$).

\smallskip\noindent\textbf{Lever 2: routing-fixed sensitivity and routing drift.}
Stability under ICC depends jointly on execution sensitivity with a fixed routed state and on routing drift. Let $\mathcal{R}_{\mathrm{base}}(\theta)$ denote routing-fixed sensitivity, $\mathcal{R}_{\mathrm{route}}(\theta)$ the routing drift, and let $L_F^{B}(D_0)$ be a uniform execution envelope over $\mathrm{supp}(D_0)$ bounding how strongly execution magnifies routing changes~\cite{khromov2023some}. Then
\begin{equation}
\mathcal{R}_{\mathrm{stab}}(\theta)
\ \le\
\mathcal{R}_{\mathrm{base}}(\theta)
+
L_F^{B}(D_0)\cdot \mathcal{R}_{\mathrm{route}}(\theta).
\label{eq:H2_stab_decomp_main_th}
\end{equation}
When routing is constant, $\mathcal{R}_{\mathrm{route}}=0$, and Eq.~\eqref{eq:H2_stab_decomp_main_th} reduces to the routing-fixed sensitivity term; the additional $L_F^{B}(D_0)\cdot\mathcal{R}_{\mathrm{route}}$ term captures the ICC-specific drift pathway. Formal definitions and proofs for Eq.~\eqref{eq:H2_ood_bound_main_th} and Eq.~\eqref{eq:H2_stab_decomp_main_th} are in Appendix~\ref{app:H2_details}.

\smallskip\noindent
These two levers directly motivate STEM-GNN: MoE improves coverage ($\beta_{\mathrm{cov}}\downarrow$), VQ stabilizes the encoder-to-head interface by absorbing small representation changes (Lemma~\ref{lem:vq_margin_invariance}), and Lipschitz regularization limits amplification of residual token changes (Prop.~\ref{prop:vq_lip_combo}); empirical diagnostics in Fig.~\ref{fig:routing_switch} confirm both routing specialization and VQ stability. The analysis relies on a Lipschitz stability constraint and a finite routed mechanism family; relaxing these assumptions is a natural direction for future work.

\section{Methodology}
\label{sec:method}

% Two-column spanning figure (use in a twocolumn paper)
\begin{figure*}[!t]
  \centering
  \includegraphics[width=\textwidth]{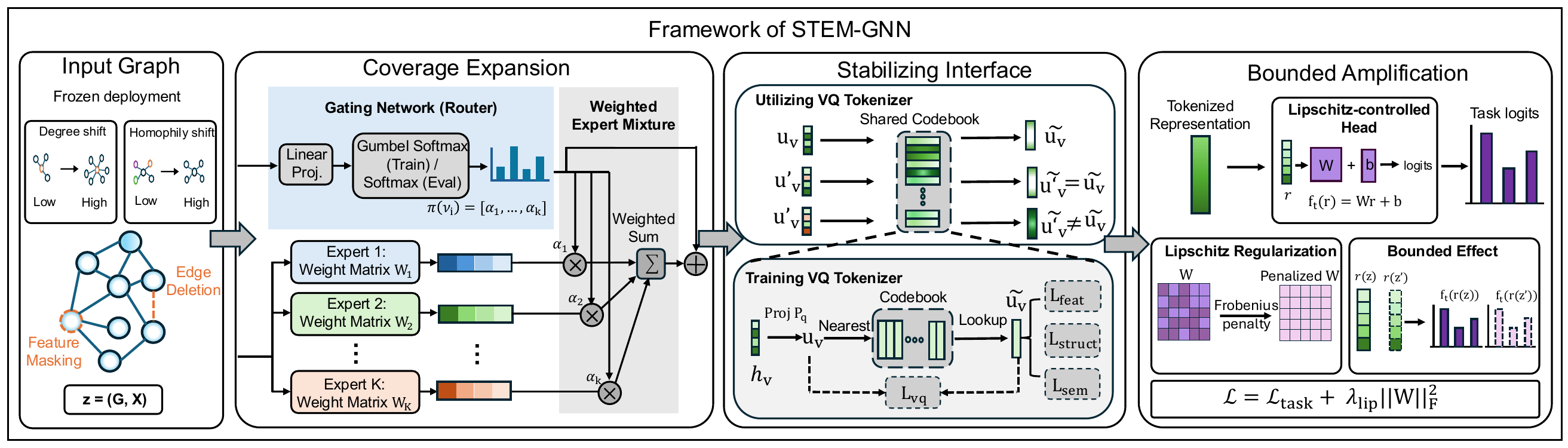}
  \vspace{-0.3in}
  \caption{The STEM-GNN framework: the MoE message-passing encoder first expands mechanism coverage via input-dependent routing, then a VQ token interface discretizes intermediate representations with a shared codebook to stabilize the encoder-to-head pathway, and finally a Lipschitz-regularized prediction head bounds amplification of residual variations.}
    \vspace{-0.15in}
  \label{fig:method_overview}
\end{figure*}

Building on this analysis, we propose \textbf{STEM-GNN}, a pretrain-then-finetune GNN framework for robust instance-conditional inference under frozen deployment (Figure~\ref{fig:method_overview}). STEM-GNN combines MoE routing for coverage expansion, VQ tokenization for a stable encoder-to-head interface, and Lipschitz regularization to control amplification of residual token changes.
% Experiments demonstrate improved OOD generalization and perturbation stability while maintaining competitive clean performance.

\subsection{Coverage Expansion via MoE Encoder}
\label{sec:method_moe}

% \vspace{3pt}\noindent\textbf{Motivation.}
Static inference applies a single computation rule to all inputs, which can 
be overly restrictive under heterogeneous deployment conditions.
We therefore use a MoE message-passing encoder to expand 
the deployed mechanism family under a fixed parameter and compute budget, 
with globally shared experts and input-dependent routing that can vary
across nodes and layers.

\vspace{3pt}\noindent\textbf{MoE message passing with soft routing.}
Let $h_v^{(l)}\in\mathbb{R}^{d_l}$ denote the representation of node $v$ 
at layer $l$.
Each MoE layer maintains $K$ shared expert operators $\{f_k^{(l)}\}_{k=1}^{K}$ 
and a lightweight router $g^{(l)}$.
The router outputs a soft routing distribution:
\begin{equation}
\pi^{(l)}(v)=\mathrm{softmax}\!\big(g^{(l)}(h_v^{(l)})\big)\in\Delta^{K-1},
\label{eq:moe_gate_soft}
\end{equation}
where $\Delta^{K-1}$ is the probability simplex over $K$ experts.
Each layer first computes a neighborhood summary:
\begin{equation}
z_v^{(l)}=\psi^{(l)}\!\Big(h_v^{(l)},
\operatorname{AGG}_{u\in\mathcal{N}(v)}\phi^{(l)}(h_u^{(l)})\Big),
\label{eq:moe_summary_generic}
\end{equation}
then updates the representation via a soft mixture of expert transforms:
\begin{equation}
h_v^{(l+1)}=
\sigma\!\Big(\sum_{k=1}^{K}\pi_k^{(l)}(v)\, f_k^{(l)}\!\big(z_v^{(l)}\big)\Big),
\label{eq:moe_update_soft}
\end{equation}
where $\sigma$ is a pointwise nonlinearity.
During training, routing is optimized with a Gumbel-Softmax relaxation 
(temperature $\tau>0$); deployment uses the deterministic softmax 
in Eq.~\eqref{eq:moe_gate_soft}.

\vspace{3pt}\noindent\textbf{Effective mechanisms and coverage.}
Eq.~\eqref{eq:moe_update_soft} realizes an input-dependent effective operator:
\begin{equation}
f_{\mathrm{eff}}^{(l)}(\cdot;v)
:=\sum_{k=1}^{K}\pi_k^{(l)}(v)\, f_k^{(l)}(\cdot),
\label{eq:moe_effective}
\end{equation}
so a single frozen parameter set realizes a family of mechanisms indexed 
by the routing distributions $\{\pi^{(l)}(v)\}$.
This enlarges the deployed mechanism family under a fixed budget, improving 
coverage across diverse deployment scenarios.

% \vspace{3pt}\noindent\textbf{Summary.}
% The MoE encoder expands coverage via input-dependent routing under frozen 
% deployment. However, routing can drift under distribution shifts and 
% perturbations, motivating the stabilized token interface introduced next.

\subsection{Stabilizing Intermediate Representations via VQ Tokenization}
\label{sec:method_vq}

% \vspace{3pt}\noindent\textbf{Motivation.}
MoE enables input-dependent execution, but the encoder outputs can still drift
under environment shifts and admissible perturbations. Under frozen deployment,
this drift directly perturbs the input to the task head. To stabilize the
encoder-to-head interface, a vector-quantized (VQ) token layer discretizes
intermediate representations so the head consumes tokens from a fixed codebook
rather than drifting continuous embeddings.
% \vspace{3pt}\noindent\textbf{Nearest-code tokenization.} 

Specifically, let $h_v \in \mathbb{R}^{d}$ be the encoder output for node $v \in V$, and let
$u_v = P_q(h_v) \in \mathbb{R}^{d_q}$ be a projected embedding. Given a codebook
$\mathcal{C} = \{c_j\}_{j=1}^{M} \subset \mathbb{R}^{d_q}$, VQ assigns a token
index and codeword by nearest-code matching:
% \begin{equation}
% q_\theta(v;z) := \arg\min_{1 \le j \le M} \|u_v(z) - c_j\|_2,
% \qquad
% Q(u_v(z)) := c_{q_\theta(v;z)}.
% \label{eq:vq_assign}
% \end{equation}
\begin{equation}
q_\theta(v;z) := \arg\min_{1 \le j \le M} \|u_v(z) - c_j\|_2,\;
Q(u_v(z)) := c_{q_\theta(v;z)}.
\label{eq:vq_assign}
\end{equation}

% \vspace{3pt}\noindent\textbf{Training protocol.}
The projection $P_q$ and codebook $\mathcal{C}$ are learned with the standard
VQ-VAE objective~\cite{van2017neural}:
\begin{equation}
\mathcal{L}_{\mathrm{vq}}
=
\mathbb{E}_{v}\!\Big[
\|\mathrm{sg}(u_v)-Q(u_v)\|_2^2
+
\beta\|u_v-\mathrm{sg}(Q(u_v))\|_2^2
\Big],
\label{eq:vq_loss}
\end{equation}
where $\mathrm{sg}(\cdot)$ denotes stop-gradient.
During pretraining, we train the encoder, $P_q$, and $\mathcal{C}$ jointly,
and apply our feature-level, structure-level, and semantic-level pretraining
signals on the tokenized representations.
During finetuning, we freeze the codebook $\mathcal{C}$ and the quantization
rule in Eq.~\eqref{eq:vq_assign} to preserve a consistent interface partition.
% \vspace{3pt}\noindent\textbf{Deployed form.}
The deployed predictor is:
\begin{equation}
f_\theta(z) := f_t\!\big(Q(u_\theta(z))\big),
\label{eq:vq_model_comp}
\end{equation}
where $f_t$ is the task head operating on tokenized embeddings.

\vspace{3pt}\noindent\textbf{Token invariance under small perturbations.}
For each node $v$, define the quantization margin at $z$:
\begin{equation}
m_v(z)
:=
\min_{j \neq q_\theta(v;z)} \|u_v(z) - c_j\|_2
-
\|u_v(z) - c_{q_\theta(v;z)}\|_2.
\label{eq:vq_margin}
\end{equation}

% \begin{lemma}[Margin-based invariance of the VQ interface]
% \label{lem:vq_margin_invariance}
% If for every $v \in V$,
% \begin{equation}
% \sup_{z' \in B(z)} \|u_v(z') - u_v(z)\|_2 \;<\; \tfrac{1}{2} m_v(z),
% \label{eq:vq_margin_condition}
% \end{equation}
% then token indices are invariant on $B(z)$, i.e., $q_\theta(v;z')=q_\theta(v;z)$
% for all $z' \in B(z)$ and all $v \in V$, and hence
% $Q(u_\theta(z')) = Q(u_\theta(z))$ for all $z' \in B(z)$.
% \end{lemma}

\begin{lemma}[Margin-based invariance of the VQ interface]
\label{lem:vq_margin_invariance}
If for every $v \in V$,
\begin{equation}
\sup_{z' \in B(z)} \|u_v(z') - u_v(z)\|_2 \;<\; \tfrac{1}{2} m_v(z),
\label{eq:vq_margin_condition}
\end{equation}
then $q_\theta(v;z') = q_\theta(v;z)$ for all $z' \in B(z)$ and $v \in V$, and thus
$Q(u_\theta(z')) = Q(u_\theta(z))$ for all $z' \in B(z)$.
\end{lemma}

\begin{proof}
The detailed proof is provided in Appendix~\ref{sec:appendix_vq_proof}.
\end{proof}

Note that routing variation is implicit in $u_v(z')-u_v(z)$, since the soft router (Eq.~\ref{eq:moe_gate_soft}) is a continuous function of $z$; the same bound therefore covers both feature drift and routing-induced drift.

% \vspace{3pt}\noindent\textbf{Summary.}
% VQ stabilizes the encoder-to-head interface by discretizing intermediate
% representations into a fixed codebook. When perturbations are small, token
% assignments remain unchanged (Lemma~\ref{lem:vq_margin_invariance}). When
% perturbations cause token switches, output variation is governed by how strongly
% the task head amplifies these discrete changes, motivating the Lipschitz
% regularization introduced next.

\subsection{Bounding Output Amplification via Lipschitz Regularization}
\label{sec:method_lip}

% \vspace{3pt}\noindent\textbf{Motivation.}
VQ stabilizes the encoder–head interface by mapping representations to a fixed codebook.
However, when perturbations cause token switches, the remaining deployment variation depends on how strongly the task head amplifies its input changes.
We thus regularize the head to bound its Lipschitz constant~\cite{miyato2018spectral}, constraining input–output amplification under shifts and perturbations.

% \vspace{3pt}\noindent\textbf{Head and regularization.}
Specifically, let $r(z)\in\mathbb{R}^{d_q}$ denote the input after tokenization (nodewise for node tasks, or after a pooling for graph tasks), and define $f_\theta(z)=f_t(r(z))$.
In our implementation, the head is linear:
\begin{equation}
f_t(r)=Wr+b.
\label{eq:linear_head_lip}
\end{equation}
Its Lipschitz constant under $\|\cdot\|_2$ equals $\|W\|_2$ and satisfies $\|W\|_2\le\|W\|_F$.
To softly control this constant during training, we add a Frobenius penalty:
\begin{equation}
\mathcal{L}
=
\mathcal{L}_{\mathrm{task}}
+\lambda_{\mathrm{lip}}\|W\|_F^{2},
\label{eq:lip_fro_penalty}
\end{equation}
which discourages large operator norms and reduces worst-case amplification through the head.

% \vspace{3pt}\noindent\textbf{Implication with tokenization.}
Let $\mathcal{C}$ be the VQ codebook and $\mathrm{diam}(\mathcal{C})=\max_{i,j}\|c_i-c_j\|_2$.
Since VQ restricts head inputs to come from a finite set of token embeddings, any token switch induces a bounded change in the head input (bounded by the codebook diameter, up to pooling for graph-level tasks).
With a controlled head Lipschitz constant, the resulting output change is bounded accordingly.

\begin{proposition}[Bounded output amplification under token switches]
\label{prop:vq_lip_combo}
Assume the head $f_t$ is $L$-Lipschitz under $\|\cdot\|_2$ and, for graph-level tasks, the pooling operator is $L_{\mathrm{pool}}$-Lipschitz with respect to the nodewise norm; set $L_{\mathrm{pool}}=1$ for node-level tasks. Then for any admissible perturbation $z'\in B(z)$, the output satisfies
\begin{equation}
\|f_t(r(z'))-f_t(r(z))\|_2
\ \le\
L\,L_{\mathrm{pool}}\,\mathrm{diam}(\mathcal{C}).
\label{eq:vq_lip_combo_body}
\end{equation}
\end{proposition}

\begin{proof}
The detailed proof is provided in Appendix~\ref{sec:appendix_vq_lip_bound}.
\end{proof}

% \vspace{3pt}\noindent\textbf{Takeaway.}
Tokenization bounds how much the head input can change when tokens switch, while Lipschitz regularization limits how much any such input change can be amplified into output variation.
Together, they control deployment-time sensitivity without changing inference-time computation.

\begin{table*}[ht]
  \centering
  % \scriptsize
  % \setlength{\tabcolsep}{3pt}
  % \renewcommand{\arraystretch}{1.05}

  \caption{Clean graph results on node, link, and graph tasks.
    Node/Link/Graph metrics are accuracy (\%); link prediction is multi-class relation classification (predicting relation type $r$ given an entity pair $(h,t)$).
    Best per column is {\setlength{\fboxsep}{1.5pt}\colorbox{bestcolor}{\textbf{bold}}}, second best is {\setlength{\fboxsep}{1.5pt}\colorbox{secondcolor}{\underline{underlined}}}.
    Results are reported as mean $\pm$ std over $k$ runs with different random seeds ($k=10$ by default, $k=20$ on \emph{WikiCS}).}
  \label{tab:clean}
  % \vspace{-2mm}
  \vspace{-0.15in}
  \resizebox{\textwidth}{!}{
    \begin{tabular}{lccccccccc}
      \toprule
                                        & \multicolumn{4}{c}{Node Classification}              & \multicolumn{2}{c}{Link Classification}              & \multicolumn{2}{c}{Graph Classification}             &                                                                                                                                                                                                                                                                                                                             \\
      \cmidrule(lr){2-5} \cmidrule(lr){6-7} \cmidrule(lr){8-9}
      Method                            & Cora                                                 & PubMed                                               & Wiki-CS                                              & Arxiv                                                & WN18RR                                               & FB15K237                                             & HIV                                                  & PCBA                                                 & Avg.                                     \\
      \midrule
      Linear                            & 58.03$_{\pm2.33}$                                    & 68.66$_{\pm2.24}$                                    & 70.36$_{\pm0.58}$                                    & 66.50$_{\pm0.14}$                                    & 78.50$_{\pm0.59}$                                    & 87.39$_{\pm0.07}$                                    & 66.37$_{\pm1.11}$                                    & 72.30$_{\pm0.34}$                                    & 71.01                                    \\
      GCN\cite{kipf2016semi}            & 75.65$_{\pm1.37}$                                    & 75.61$_{\pm2.10}$                                    & 75.28$_{\pm1.34}$                                    & \cellcolor{secondcolor}\underline{71.40}$_{\pm0.08}$ & 73.79$_{\pm0.39}$                                    & 82.22$_{\pm0.28}$                                    & 64.84$_{\pm4.78}$                                    & 71.32$_{\pm0.49}$                                    & 73.76                                    \\
      GAT\cite{velickovic2017graph}     & 76.24$_{\pm1.62}$                                    & 74.86$_{\pm1.87}$                                    & 76.78$_{\pm0.78}$                                    & 70.87$_{\pm0.24}$                                    & 80.16$_{\pm0.27}$                                    & 88.93$_{\pm0.15}$                                    & 65.54$_{\pm6.93}$                                    & 70.12$_{\pm0.89}$                                    & 75.44                                    \\
      GIN\cite{xu2018powerful}          & 73.59$_{\pm2.10}$                                    & 69.51$_{\pm6.87}$                                    & 49.77$_{\pm4.72}$                                    & 65.05$_{\pm0.50}$                                    & 74.02$_{\pm0.55}$                                    & 83.21$_{\pm0.53}$                                    & 66.86$_{\pm3.48}$                                    & 72.69$_{\pm0.22}$                                    & 69.34                                    \\
      \midrule
      DGI\cite{velivckovic2018deep}     & 72.10$_{\pm0.34}$                                    & 73.13$_{\pm0.64}$                                    & 75.32$_{\pm0.95}$                                    & 69.15$_{\pm0.20}$                                    & 75.75$_{\pm0.90}$                                    & 81.34$_{\pm0.15}$                                    & 59.62$_{\pm1.21}$                                    & 63.31$_{\pm0.89}$                                    & 71.22                                    \\
      BGRL\cite{thakoor2021large}       & 71.20$_{\pm0.30}$                                    & 75.29$_{\pm1.33}$                                    & 76.53$_{\pm0.69}$                                    & 71.19$_{\pm0.20}$                                    & 75.44$_{\pm0.30}$                                    & 80.66$_{\pm0.29}$                                    & 63.95$_{\pm1.06}$                                    & 67.09$_{\pm1.00}$                                    & 72.67                                    \\
      GraphMAE\cite{hou2022graphmae}    & 73.10$_{\pm0.40}$                                    & 74.32$_{\pm0.33}$                                    & 77.61$_{\pm0.39}$                                    & 70.90$_{\pm0.31}$                                    & 78.99$_{\pm0.48}$                                    & 85.30$_{\pm0.16}$                                    & 61.04$_{\pm0.55}$                                    & 63.30$_{\pm0.78}$                                    & 73.07                                    \\
      GIANT\cite{chien2021node}         & 75.13$_{\pm0.49}$                                    & 72.31$_{\pm0.53}$                                    & 76.56$_{\pm0.88}$                                    & 70.10$_{\pm0.32}$                                    & 84.36$_{\pm0.30}$                                    & 87.45$_{\pm0.54}$                                    & 65.44$_{\pm1.39}$                                    & 61.49$_{\pm0.99}$                                    & 74.11                                    \\
      GFT\cite{wang2024gft}             & 77.83$_{\pm1.17}$                                    & \cellcolor{secondcolor}\underline{77.70}$_{\pm1.39}$ & 78.40$_{\pm0.70}$                                    & 69.17$_{\pm0.57}$                                    & \cellcolor{secondcolor}\underline{91.19}$_{\pm0.33}$ & \cellcolor{secondcolor}\underline{89.91}$_{\pm0.05}$ & \cellcolor{secondcolor}\underline{70.93}$_{\pm1.02}$ & \cellcolor{secondcolor}\underline{78.95}$_{\pm0.27}$ & \cellcolor{secondcolor}\underline{79.26} \\
      \midrule
      CaNet\cite{wu2024graph}           & 76.41$_{\pm1.48}$                                    & 75.33$_{\pm1.54}$                                    & \cellcolor{secondcolor}\underline{78.88}$_{\pm0.45}$ & 66.70$_{\pm0.32}$                                    & --                                                   & --                                                   & --                                                   & --                                                   & --                                       \\
      GraphMETRO\cite{wu2024graphmetro} & 75.69$_{\pm3.34}$                                    & 75.24$_{\pm1.68}$                                    & 74.59$_{\pm2.08}$                                    & 67.27$_{\pm0.35}$                                    & --                                                   & --                                                   & --                                                   & --                                                   & --                                       \\
      MARIO\cite{zhu2024mario}          & \cellcolor{secondcolor}\underline{77.85}$_{\pm1.31}$ & 77.12$_{\pm0.98}$                                    & 78.60$_{\pm0.44}$                                    & 67.57$_{\pm0.45}$                                    & --                                                   & --                                                   & --                                                   & --                                                   & --                                       \\
      TFEGNN\cite{duan2024unifying}     & 77.33$_{\pm1.47}$                                    & 77.04$_{\pm1.02}$                                    & 76.74$_{\pm0.64}$                                    & 68.56$_{\pm0.67}$                                    & --                                                   & --                                                   & --                                                   & --                                                   & --                                       \\
      \midrule
      % \rowcolor{ourscolor}
      \textbf{STEM-GNN}                     & \cellcolor{bestcolor}\textbf{79.53}$_{\pm1.32}$      & \cellcolor{bestcolor}\textbf{77.84}$_{\pm1.66}$      & \cellcolor{bestcolor}\textbf{80.11}$_{\pm0.53}$      & \cellcolor{bestcolor}\textbf{72.31}$_{\pm0.25}$      & \cellcolor{bestcolor}\textbf{92.34}$_{\pm0.25}$      & \cellcolor{bestcolor}\textbf{90.26}$_{\pm0.16}$      & \cellcolor{bestcolor}\textbf{73.54}$_{\pm1.02}$      & \cellcolor{bestcolor}\textbf{80.39}$_{\pm0.38}$      & \cellcolor{bestcolor}\textbf{80.79}      \\
      \bottomrule
    \end{tabular}
  }
  \vspace{-2mm}
\end{table*}

\section{Experiments}

\subsection{Experimental Setup}

% Fig 3 (ablation) + Fig 4 (stability) merged into one two-column-spanning row,
% side by side, each keeping its own figure number/label (via \captionof).
\begin{figure*}[!t]
  \centering
  \begin{minipage}[t]{0.49\textwidth}
    \centering
    \includegraphics[width=\linewidth]{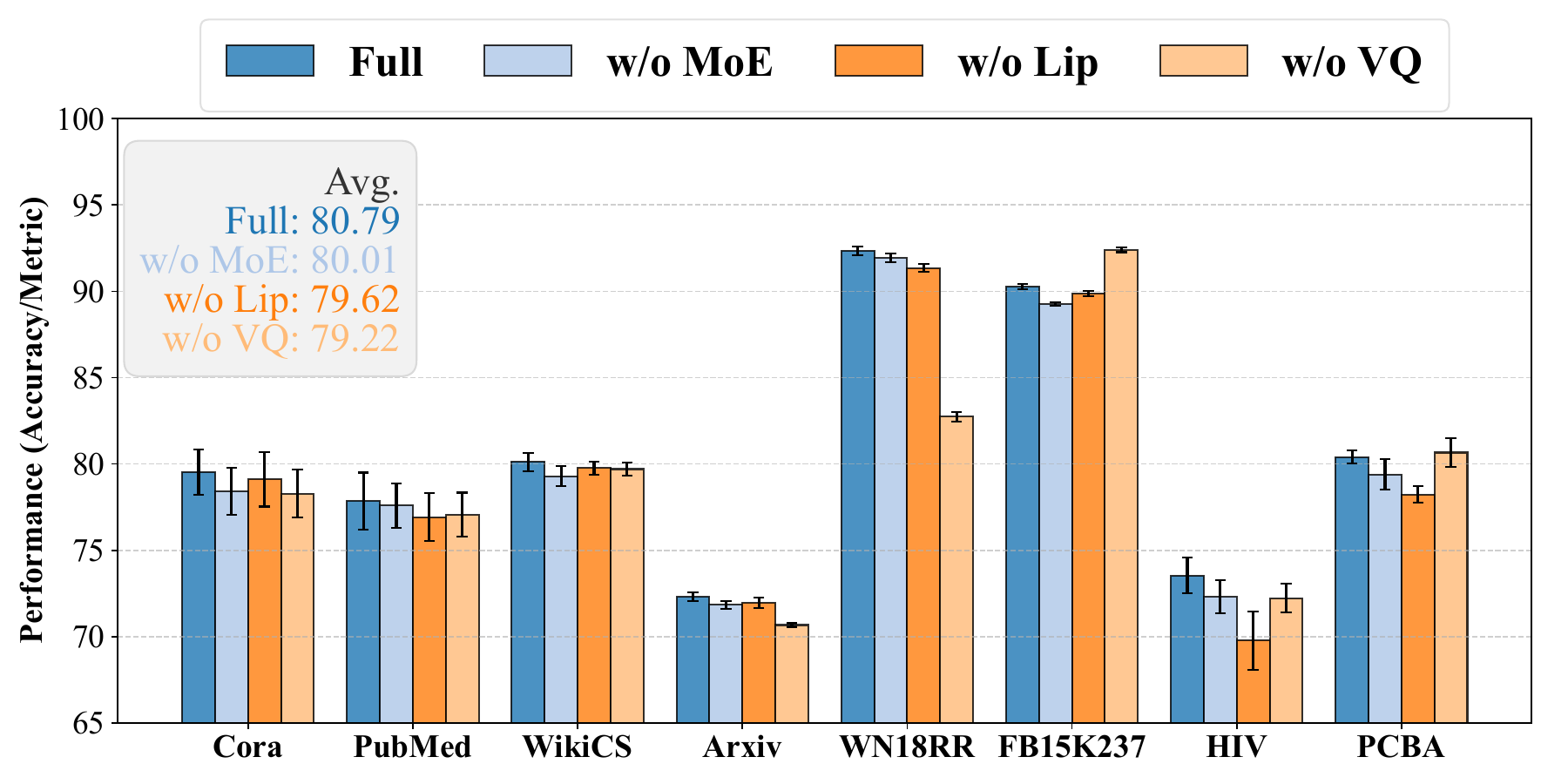}
    \vspace{-0.15in}
    \captionof{figure}{Ablation study on clean original graphs (mean $\pm$ std). We compare the full STEM-GNN with variants removing MoE (w/o MoE), Lipschitz regularization (w/o Lip), or VQ (w/o VQ) across all benchmarks, and report the average performance across datasets.}
    \label{fig:clean_graph_ablation}
  \end{minipage}\hfill
  \begin{minipage}[t]{0.49\textwidth}
    \centering
    \includegraphics[width=\linewidth]{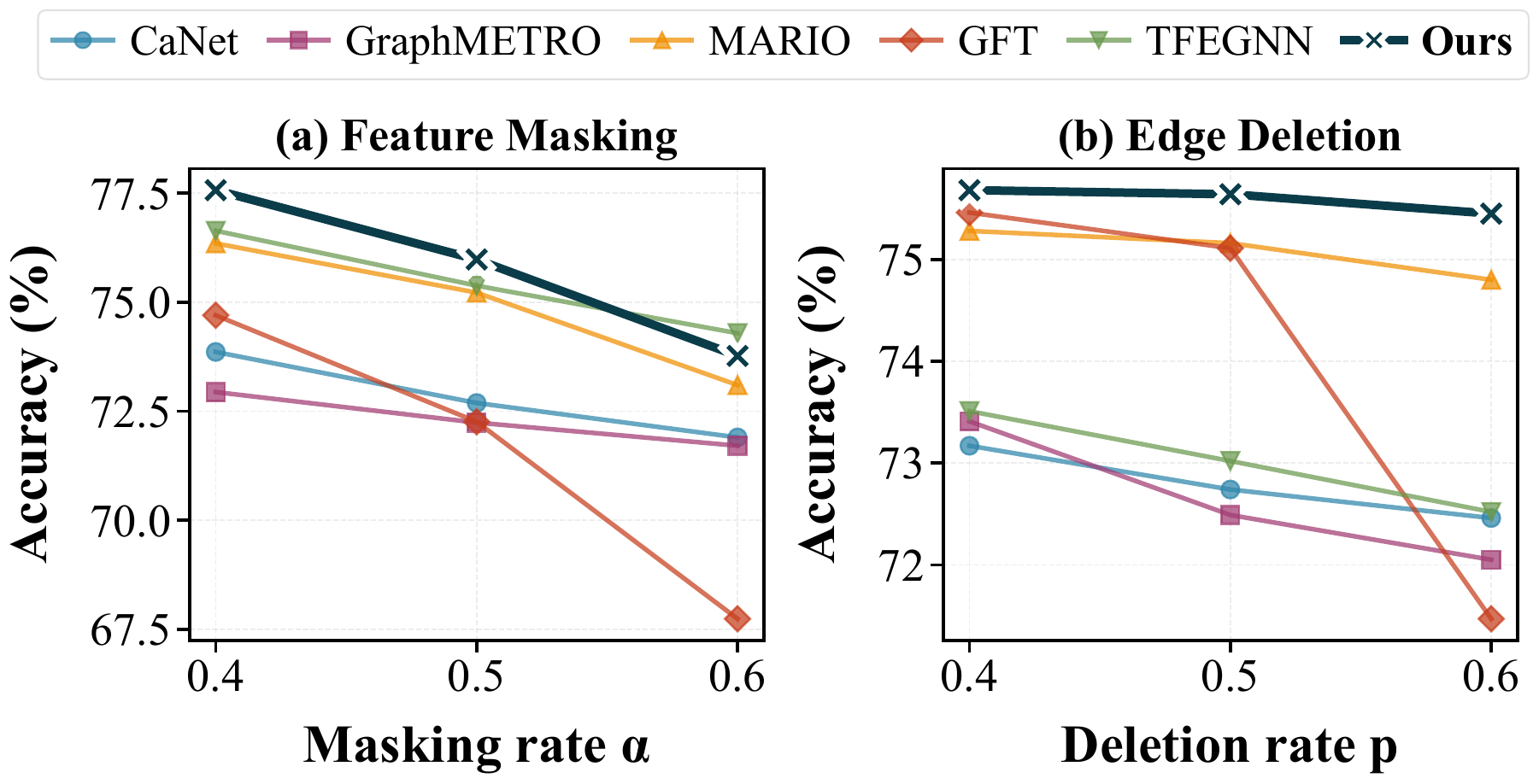}
    \vspace{-0.15in}
    \captionof{figure}{Stability under inference-time perturbation on representative node benchmarks. (a) Feature masking on Cora (rate $\alpha$). (b) Edge deletion on PubMed (rate $p$).}
    \label{fig:noise_cora_pubmed_feature_structural}
  \end{minipage}
\end{figure*}

\noindent\textbf{Datasets.}
We evaluate on eight benchmarks spanning node, link, and graph tasks across diverse domains:
(i) citation networks (Cora, PubMed, and Arxiv) and a Wikipedia graph (WikiCS) for node classification;
(ii) knowledge graphs (WN18RR and FB15K-237) for link-level multi-class relation prediction;
(iii) molecular graphs (HIV and PCBA) for graph classification.
We use standard public splits for Cora/PubMed, the official 20 splits for WikiCS, official benchmark splits for Arxiv and the molecular datasets, and fixed splits for FB15K-237/WN18RR following~\cite{wang2024gft,liu2023one}.

\vspace{3pt}\noindent\textbf{Baselines.}
We compare STEM-GNN with various baselines from three families:
(i) supervised baselines: Linear, GCN~\cite{kipf2016semi}, GAT~\cite{velickovic2017graph}, and GIN~\cite{xu2018powerful};
(ii) self/weakly-supervised pretraining and foundation methods: DGI~\cite{velivckovic2018deep}, BGRL~\cite{thakoor2021large}, GraphMAE~\cite{hou2022graphmae}, GIANT~\cite{chien2021node}, and GFT~\cite{wang2024gft};
(iii) OOD- and robustness-oriented methods: CaNet ~\cite{wu2024graph}, GraphMETRO~\cite{wu2024graphmetro}, MARIO~\cite{zhu2024mario}, and TFEGNN~\cite{duan2024unifying}.

\vspace{3pt}\noindent\textbf{Evaluation Protocol.}
Across all experiments, models are trained on clean graphs and evaluated under frozen deployment, with no test-time updates or adaptation. For baselines, we follow the authors' recommended setup (official code when available) and run all methods under a matched protocol. Clean evaluation spans all benchmarks; node-level frozen-deployment studies additionally evaluate stability, OOD generalization, and tri-objective balance. \emph{(i) Clean.} \textit{Fit} is clean test accuracy on the corresponding standard or ID-test split. \emph{(ii) Stability.} Inference-time perturbations apply independent Bernoulli feature masking (masked entries zeroed) and undirected-edge deletion on evaluation nodes. \emph{(iii) OOD generalization.} Distribution shifts sort nodes by undirected degree or by mean cosine similarity to one-hop neighbors; the bottom/top $15\%$ buckets are OOD-low/OOD-high, the remaining valid nodes form ID with class-stratified train/val/test splits, and \textit{OOD-worst} is the minimum accuracy across OOD buckets. \emph{(iv) Tri-objective.} For the more granular tri-objective evaluation, nodes are sorted by feature-homophily score (mean one-hop feature cosine similarity) into three OOD buckets---bottom $10\%$ (OOD3), $10$--$20\%$ (OOD2), and $20$--$30\%$ (OOD1)---with the mid-range $30$--$80\%$ as ID (the top $20\%$ is held out); for each run, ID is re-split with class stratification, and \textit{Perturb-mean} averages ID-test accuracy under feature masking at rates $\alpha\in\{0.2,0.4,0.6,0.8\}$. Results are averaged over $10$ random seeds unless otherwise specified.

\vspace{-0.1in}
\subsection{Results on Clean Graphs}

% Fig 4 (noise) moved into the merged two-column figure* above (side by side with Fig 3).
% \begin{figure}[!t]
%   \centering
%   \makebox[\linewidth][c]{%
%     \includegraphics[width=1.0\linewidth]{figures/noise.pdf}
%   }
%       \vspace{-0.25in}
%   \caption{Stability under inference-time perturbation on representative node benchmarks. (a) Feature masking on Cora (rate $\alpha$). (b) Edge deletion on PubMed (rate $p$).}
%       \vspace{-0.2in}
% \label{fig:noise_cora_pubmed_feature_structural}
% \end{figure}

\noindent\textbf{Main Results.}
Table~\ref{tab:clean} reports clean-graph performance across node, link, and graph tasks.
STEM-GNN achieves the best average (80.79\%), ranking first on all eight benchmarks, showing that clean accuracy remains compatible with deployment-oriented design under frozen deployment.
Compared to GFT (79.26\%), STEM-GNN improves most on node (+1.70 Cora, +3.14 Arxiv) and graph tasks (+2.61 HIV, +1.44 PCBA), while remaining competitive on link tasks (+1.15 WN18RR, +0.35 FB15K237).
This pattern suggests an advantage on higher-heterogeneity settings (node/graph), where a single fixed computation is more likely to be suboptimal across instances.
We also observe smaller standard deviations on several datasets (e.g., $\pm$0.25 vs.\ $\pm$0.57 on Arxiv), indicating reduced sensitivity to random seeds.
Against OOD/robust baselines on node benchmarks, STEM-GNN remains competitive while delivering strong link and graph results; the following sections show that the same frozen model preserves these gains under distribution shifts and perturbations.

\begin{table*}[t]
        \caption{Node classification accuracy (\%) under attribute-defined distribution shifts.
\textbf{Degree shift} and \textbf{homophily shift} evaluate models trained/selected on a class-stratified split of the in-distribution nodes and tested on buckets defined by the corresponding attribute.
OOD-low/OOD-high denote the bottom/top 15\% test nodes by the shift attribute.}

        \vspace{-0.1in}
        \label{tab:distribution-shift}
        \centering
        % \setlength{\tabcolsep}{4pt}
        % \renewcommand{\arraystretch}{1.12}
        % \scriptsize
        \resizebox{0.8\linewidth}{!}{
                \begin{tabular}{cl llcccccc}
                        \toprule
                         & Dataset                           & Bucket
                         & CaNet\cite{wu2024graph}
                         & GraphMETRO\cite{wu2024graphmetro}
                         & MARIO\cite{zhu2024mario}
                         & GFT\cite{wang2024gft}
                         & TFEGNN\cite{duan2024unifying}
                         & \textbf{STEM-GNN}                                                                                                                                                                                                                                                                                                                                                                                   \\
                        \midrule
                        \multirow{10}{*}{\rotatebox[origin=c]{90}{\textbf{Degree Shift}}}
                         & \multirow{3}{*}{Cora}
                         & ID                                & 84.12$_{\pm1.54}$ & 83.78$_{\pm1.19}$                                    & \cellcolor{secondcolor}\underline{87.07}$_{\pm1.81}$ & 87.01$_{\pm1.48}$                                    & 86.63$_{\pm1.45}$                                    & \cellcolor{bestcolor}\textbf{87.39}$_{\pm1.14}$                                                             \\
                         &                                   & OOD-low           & 76.35$_{\pm1.63}$                                    & 77.00$_{\pm1.65}$                                    & 78.77$_{\pm1.03}$                                    & 81.22$_{\pm0.70}$                                    & \cellcolor{bestcolor}\textbf{83.00}$_{\pm0.82}$      & \cellcolor{secondcolor}\underline{82.73}$_{\pm0.85}$ \\
                         &                                   & OOD-high          & 84.38$_{\pm0.83}$                                    & 82.22$_{\pm2.99}$                                    & \cellcolor{secondcolor}\underline{85.71}$_{\pm0.71}$ & 84.98$_{\pm1.11}$                                    & \cellcolor{secondcolor}\underline{85.71}$_{\pm1.79}$ & \cellcolor{bestcolor}\textbf{86.08}$_{\pm1.02}$      \\
                        \cmidrule{2-9}
                         & \multirow{3}{*}{PubMed}
                         & ID                                & 89.07$_{\pm0.48}$ & 87.93$_{\pm0.42}$                                    & 84.97$_{\pm0.68}$                                    & 89.29$_{\pm0.53}$                                    & \cellcolor{secondcolor}\underline{89.65}$_{\pm0.47}$ & \cellcolor{bestcolor}\textbf{90.61}$_{\pm0.35}$                                                             \\
                         &                                   & OOD-low           & 88.85$_{\pm0.14}$                                    & 87.15$_{\pm0.22}$                                    & 83.50$_{\pm0.52}$                                    & 88.73$_{\pm0.46}$                                    & \cellcolor{secondcolor}\underline{89.14}$_{\pm0.59}$ & \cellcolor{bestcolor}\textbf{89.44}$_{\pm0.42}$      \\
                         &                                   & OOD-high          & 88.75$_{\pm0.49}$                                    & 88.26$_{\pm0.31}$                                    & 85.86$_{\pm0.33}$                                    & \cellcolor{secondcolor}\underline{88.95}$_{\pm0.39}$ & 86.38$_{\pm0.48}$                                    & \cellcolor{bestcolor}\textbf{89.95}$_{\pm0.22}$      \\
                        \cmidrule{2-9}
                         & \multirow{3}{*}{WikiCS}
                         & ID                                & 83.48$_{\pm0.63}$ & 83.67$_{\pm1.41}$                                    & \cellcolor{secondcolor}\underline{85.56}$_{\pm0.83}$ & 85.07$_{\pm0.53}$                                    & 81.55$_{\pm0.97}$                                    & \cellcolor{bestcolor}\textbf{86.03}$_{\pm0.49}$                                                             \\
                         &                                   & OOD-low           & 75.68$_{\pm0.45}$                                    & 75.99$_{\pm1.53}$                                    & 74.64$_{\pm0.95}$                                    & 76.01$_{\pm1.12}$                                    & \cellcolor{bestcolor}\textbf{78.81}$_{\pm0.69}$      & \cellcolor{secondcolor}\underline{77.26}$_{\pm0.77}$ \\
                         &                                   & OOD-high          & \cellcolor{secondcolor}\underline{86.17}$_{\pm0.53}$ & 81.38$_{\pm2.45}$                                    & 83.45$_{\pm0.76}$                                    & 85.80$_{\pm0.74}$                                    & 85.18$_{\pm0.87}$                                    & \cellcolor{bestcolor}\textbf{86.23}$_{\pm1.56}$      \\
                        \cmidrule{2-9}
                         & Avg                               &                   & 84.09                                                & 83.04                                                & 83.28                                                & \cellcolor{secondcolor}\underline{85.23}             & 85.12                                                & \cellcolor{bestcolor}\textbf{86.19}                  \\
                        \midrule
                        \multirow{10}{*}{\rotatebox[origin=c]{90}{\textbf{Homophily Shift}}}
                         & \multirow{3}{*}{Cora}
                         & ID                                & 83.97$_{\pm0.64}$ & 84.52$_{\pm1.68}$                                    & 85.65$_{\pm1.21}$                                    & \cellcolor{secondcolor}\underline{87.20}$_{\pm1.90}$ & 86.11$_{\pm1.20}$                                    & \cellcolor{bestcolor}\textbf{87.56}$_{\pm1.23}$                                                             \\
                         &                                   & OOD-low           & 79.41$_{\pm1.16}$                                    & 78.08$_{\pm1.63}$                                    & \cellcolor{secondcolor}\underline{81.48}$_{\pm0.96}$ & 77.83$_{\pm0.86}$                                    & \cellcolor{bestcolor}\textbf{81.53}$_{\pm1.65}$      & 80.12$_{\pm1.13}$                                    \\
                         &                                   & OOD-high          & 85.52$_{\pm1.60}$                                    & 84.81$_{\pm1.57}$                                    & 87.44$_{\pm1.66}$                                    & \cellcolor{secondcolor}\underline{90.44}$_{\pm0.93}$ & 88.28$_{\pm0.89}$                                    & \cellcolor{bestcolor}\textbf{91.03}$_{\pm0.81}$      \\
                        \cmidrule{2-9}
                         & \multirow{3}{*}{PubMed}
                         & ID                                & 88.15$_{\pm0.27}$ & 87.38$_{\pm0.51}$                                    & 87.55$_{\pm0.42}$                                    & \cellcolor{secondcolor}\underline{89.70}$_{\pm0.58}$ & 89.08$_{\pm0.41}$                                    & \cellcolor{bestcolor}\textbf{90.49}$_{\pm0.43}$                                                             \\
                         &                                   & OOD-low           & 85.19$_{\pm0.34}$                                    & 84.90$_{\pm0.52}$                                    & 82.24$_{\pm0.23}$                                    & 85.52$_{\pm0.45}$                                    & \cellcolor{bestcolor}\textbf{87.34}$_{\pm0.54}$      & \cellcolor{secondcolor}\underline{86.26}$_{\pm0.43}$ \\
                         &                                   & OOD-high          & 87.08$_{\pm0.20}$                                    & 86.04$_{\pm0.61}$                                    & 85.86$_{\pm0.69}$                                    & \cellcolor{secondcolor}\underline{92.41}$_{\pm0.33}$ & 88.25$_{\pm0.28}$                                    & \cellcolor{bestcolor}\textbf{92.85}$_{\pm0.39}$      \\
                        \cmidrule{2-9}
                         & \multirow{3}{*}{WikiCS}
                         & ID                                & 83.72$_{\pm0.49}$ & \cellcolor{secondcolor}\underline{85.39}$_{\pm0.73}$ & 83.57$_{\pm0.44}$                                    & 84.64$_{\pm0.60}$                                    & 83.60$_{\pm0.60}$                                    & \cellcolor{bestcolor}\textbf{85.77}$_{\pm0.53}$                                                             \\
                         &                                   & OOD-low           & \cellcolor{secondcolor}\underline{81.60}$_{\pm0.48}$ & \cellcolor{bestcolor}\textbf{82.39}$_{\pm0.71}$      & 80.31$_{\pm0.61}$                                    & 80.08$_{\pm0.61}$                                    & 81.55$_{\pm0.63}$                                    & 80.87$_{\pm0.40}$                                    \\
                         &                                   & OOD-high          & 83.36$_{\pm0.44}$                                    & 87.95$_{\pm0.53}$                                    & 85.47$_{\pm0.25}$                                    & \cellcolor{secondcolor}\underline{89.42}$_{\pm0.52}$ & 85.59$_{\pm0.40}$                                    & \cellcolor{bestcolor}\textbf{89.88}$_{\pm0.54}$      \\
                        \cmidrule{2-9}
                         & Avg                               &                   & 84.22                                                & 84.61                                                & 84.40                                                & \cellcolor{secondcolor}\underline{86.36}             & 85.70                                                & \cellcolor{bestcolor}\textbf{87.20}                  \\
                        \bottomrule
                \end{tabular}
        }
     \vspace{-0.1in}
\end{table*}

\vspace{3pt}
\noindent\textbf{Ablation Study.}
Figure~\ref{fig:clean_graph_ablation} shows that the full model achieves the best average on clean graphs (80.79\%), and each ablation lowers the average (w/o MoE: 80.01\%, w/o Lip: 79.62\%, w/o VQ: 79.22\%; drops of $0.78$/$1.17$/$1.57$ points, respectively), indicating non-redundant benefits; on FB15K237 with $237$ fine-grained relation classes, VQ's discretization may introduce a task-specific information bottleneck that slightly lowers clean accuracy ($90.26\%$ vs.\ $92.39\%$ w/o VQ), while STEM-GNN remains competitive.
Among them, MoE yields the most uniform improvements across node/link/graph benchmarks, while VQ and Lipschitz regularization are more task-dependent, with larger gains on higher-heterogeneity or pooled graph settings, consistent with stabilizing the interface and controlling amplification.
Notably, the full model avoids sharp, dataset-specific regressions seen in some ablations, suggesting that the components complement each other by trading off coverage (MoE) against sensitivity control (VQ+Lip) rather than providing interchangeable capacity.
This complementarity helps explain why gains are broad across tasks: the same pretrain-then-finetune pipeline remains strong on clean data while retaining headroom for robustness under shifts and perturbations evaluated next.

\subsection{Results on Noisy Graphs}
We train on clean graphs and use clean validation for model selection, then evaluate the same frozen model under inference-time perturbations:
(a) feature masking ($\alpha\in\{0.4,0.5,0.6\}$) and (b) edge deletion ($p\in\{0.4,0.5,0.6\}$).
Figure~\ref{fig:noise_cora_pubmed_feature_structural} shows that STEM-GNN degrades most gracefully and maintains the strongest accuracy trend across both corruption types.
Notably, the margin widens at higher severities, where several baselines exhibit sharper drops, suggesting improved resilience to partial attribute loss and connectivity disruption without noise-aware training.
This behavior is consistent with VQ stabilizing the encoder-to-head interface under small perturbations (Lemma~\ref{lem:vq_margin_invariance}) and Lipschitz control limiting amplification when token switches occur (Proposition~\ref{prop:vq_lip_combo}).
Complete results are in Appendix~\ref{sec:appendix_experiments}.

\subsection{Out-Of-Distribution Generalization}

We evaluate OOD generalization by training on a class-stratified split of the in-distribution (non-shifted) nodes and using its clean validation accuracy for model selection only;
the same frozen model is then tested on attribute-defined distribution shifts.
Table~\ref{tab:distribution-shift} reports results under degree shift and homophily shift,
where OOD-low/OOD-high denote the bottom/top 15\% by the respective attribute.
STEM-GNN achieves the best average under both shifts (86.19\% degree, 87.20\% homophily), with the gains concentrating on challenging OOD-high buckets,
e.g., degree OOD-high on Cora (86.08 vs.\ GFT's 84.98) and homophily OOD-high on Cora (91.03 vs.\ GFT's 90.44).
This trend is consistent with stronger benefits when neighborhood aggregation is more influential, where stable routing and a discretized interface help mitigate shift-induced variation.
Overall, the results are consistent with the coverage--selection perspective of Eq.~\ref{eq:H2_ood_bound_main_th}, suggesting that expanded mechanism coverage and more stable routing under shift help explain the observed worst-environment gains.

\subsection{Transferability Analysis}

\begin{figure}[!t]
  \centering
  \makebox[\linewidth][c]{%
    \includegraphics[width=1\linewidth]{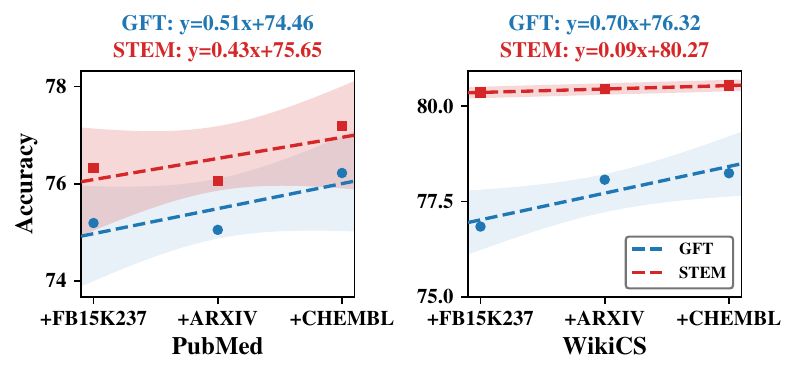}
  }
  \vspace{-0.3in}
  \caption{Cross-domain transfer with increasing pretraining diversity.
  % (FB15K-237 $\rightarrow$ +Arxiv $\rightarrow$ +ChEMBL).
STEM-GNN is less sensitive to the source mixture than GFT on PubMed and WikiCS (least-squares fits shown).
}
  \vspace{-0.1in}
\label{fig:cross-domain}
\end{figure}

Figure~\ref{fig:cross-domain} evaluates cross-domain transfer by increasing pretraining diversity (FB15K-237 $\rightarrow$ +Arxiv $\rightarrow$ +ChEMBL) and fine-tuning on PubMed and WikiCS.
Beyond mean accuracy, STEM-GNN shows stronger \emph{mixture robustness}: it is less sensitive to the source mixture and exhibits flatter scaling than GFT (least-squares slopes: PubMed 0.43 vs.\ 0.51; WikiCS 0.09 vs.\ 0.70).
Thus, adding heterogeneous sources yields more predictable gains for STEM-GNN, while GFT shows larger composition-induced variability, especially on WikiCS.
% This is valuable under frozen deployment, where practitioners must commit to a fixed pretraining recipe without exhaustive tuning.
The pattern aligns with MoE expanding coverage across heterogeneous sources and Lipschitz control limiting amplification of source-induced shifts, reducing negative transfer when mixing disparate domains.

\begin{table}[t]
  \centering
  % \footnotesize
  % \setlength{\tabcolsep}{2pt}
  % \renewcommand{\arraystretch}{1.05}
 \caption{Tri-objective results. We select the epoch by clean validation, evaluate the frozen checkpoint on \textbf{Fit}, \textbf{OOD-worst}, and \textbf{Perturb-mean} (feature masking). \textbf{Avg} is their mean.}

    \vspace{-5pt}
  \label{tab:tri_objective_multi_dataset}
  \resizebox{\columnwidth}{!}{%
    \begin{tabular}{@{}llcccc@{}}
      \toprule
      Dataset & Method                             & Fit$\uparrow$                                     & OOD-worst$\uparrow$                               & Perturb-mean$\uparrow$                            & Avg$\uparrow$                                     \\
      \midrule

      \multirow{6}{*}{Cora}
              & CaNet~\cite{wu2024graph}           & 83.63$\pm$1.86                                    & 75.20$\pm$1.80                                    & 82.53$\pm$1.76                                    & 80.45$\pm$1.81                                    \\
              & GraphMETRO~\cite{wu2024graphmetro} & 83.67$\pm$1.46                                    & \cellcolor{bestcolor}\textbf{79.11$\pm$1.14}      & 82.95$\pm$1.43                                    & 81.91$\pm$1.34                                    \\
              & MARIO~\cite{zhu2024mario}          & 83.39$\pm$1.61                                    & 78.27$\pm$1.83                                    & 83.11$\pm$1.49                                    & 81.59$\pm$1.64                                    \\
              & GFT~\cite{wang2024gft}             & \cellcolor{secondcolor}\underline{86.58$\pm$1.75} & 76.47$\pm$1.59                                    & \cellcolor{secondcolor}\underline{84.14$\pm$2.34} & \cellcolor{secondcolor}\underline{82.40$\pm$1.89} \\
              & TFEGNN~\cite{duan2024unifying}     & 83.77$\pm$0.98                                    & 77.07$\pm$1.65                                    & 83.57$\pm$1.20                                    & 81.47$\pm$1.28                                    \\
      % \rowcolor{ourscolor}
              & \textbf{STEM-GNN}                      & \cellcolor{bestcolor}\textbf{87.31$\pm$1.56}      & \cellcolor{secondcolor}\underline{78.45$\pm$0.77} & \cellcolor{bestcolor}\textbf{85.88$\pm$1.38}      & \cellcolor{bestcolor}\textbf{83.88$\pm$1.24}      \\
      \midrule

      \multirow{6}{*}{PubMed}
              & CaNet~\cite{wu2024graph}           & 88.43$\pm$0.42                                    & \cellcolor{bestcolor}\textbf{85.88$\pm$0.44}      & 86.08$\pm$0.21                                    & \cellcolor{secondcolor}\underline{86.80$\pm$0.36} \\
              & GraphMETRO~\cite{wu2024graphmetro} & 87.63$\pm$0.70                                    & 84.30$\pm$0.49                                    & 83.36$\pm$0.97                                    & 85.10$\pm$0.72                                    \\
              & MARIO~\cite{zhu2024mario}          & 86.14$\pm$0.56                                    & 84.13$\pm$0.57                                    & 83.55$\pm$0.63                                    & 84.61$\pm$0.59                                    \\
              & GFT~\cite{wang2024gft}             & \cellcolor{secondcolor}\underline{89.49$\pm$0.65} & 84.60$\pm$0.39                                    & 84.68$\pm$1.77                                    & 86.26$\pm$0.94                                    \\
              & TFEGNN~\cite{duan2024unifying}     & 88.52$\pm$0.65                                    & \cellcolor{secondcolor}\underline{85.12$\pm$0.57} & \cellcolor{bestcolor}\textbf{86.50$\pm$0.32}      & 86.71$\pm$0.51                                    \\
      % \rowcolor{ourscolor}
              & \textbf{STEM-GNN}                      & \cellcolor{bestcolor}\textbf{90.14$\pm$0.62}      & 84.77$\pm$0.50                                    & \cellcolor{secondcolor}\underline{86.26$\pm$0.66} & \cellcolor{bestcolor}\textbf{87.06$\pm$0.59}      \\
      \midrule

      \multirow{6}{*}{WikiCS}
              & CaNet~\cite{wu2024graph}           & 83.13$\pm$0.98                                    & \cellcolor{secondcolor}\underline{79.86$\pm$0.43} & 81.09$\pm$0.86                                    & 81.36$\pm$0.76                                    \\
              & GraphMETRO~\cite{wu2024graphmetro} & 81.77$\pm$1.46                                    & 76.94$\pm$1.23                                    & 78.53$\pm$1.28                                    & 79.08$\pm$1.32                                    \\
              & MARIO~\cite{zhu2024mario}          & 83.25$\pm$0.86                                    & 77.28$\pm$0.93                                    & 81.21$\pm$0.89                                    & 80.58$\pm$0.89                                    \\
              & GFT~\cite{wang2024gft}             & \cellcolor{secondcolor}\underline{85.41$\pm$0.96} & 78.46$\pm$0.79                                    & \cellcolor{secondcolor}\underline{82.44$\pm$1.12} & 82.10$\pm$0.96                                    \\
              & TFEGNN~\cite{duan2024unifying}     & 83.91$\pm$0.32                                    & \cellcolor{bestcolor}\textbf{80.84$\pm$0.53}      & 82.21$\pm$0.28                                    & \cellcolor{secondcolor}\underline{82.32$\pm$0.38} \\
      % \rowcolor{ourscolor}
              & \textbf{STEM-GNN}                      & \cellcolor{bestcolor}\textbf{85.57$\pm$0.76}      & 79.02$\pm$0.67                                    & \cellcolor{bestcolor}\textbf{83.90$\pm$0.69}      & \cellcolor{bestcolor}\textbf{82.83$\pm$0.71}      \\
      \bottomrule
    \end{tabular}%
  }
       \vspace{-0.2in}
\end{table}

\vspace{-0.1in}
\subsection{Tri-Objective results}

\noindent\textbf{Main Results.}
Table~\ref{tab:tri_objective_multi_dataset} reports a unified tri-objective evaluation on three node benchmarks.
Each method is trained once, selected by clean validation accuracy, then evaluated on:
(i) \textit{Fit} (clean test accuracy),
(ii) \textit{OOD-worst} (minimum accuracy over homophily-shifted splits), and
(iii) \textit{Perturb-mean} (mean accuracy under feature masking, $\alpha\in\{0.2,0.4,0.6,0.8\}$),
with \textit{Avg} averaging the three.
STEM-GNN achieves the best \textit{Avg} on all datasets (Cora: 83.88, PubMed: 87.06, WikiCS: 82.83), indicating a better balance under frozen deployment.
Relative to GFT, the gain improves \textit{Fit}, \textit{Perturb-mean}, and \textit{OOD-worst} rather than trading one axis for another.
On Cora, STEM-GNN improves over GFT on all three axes (\textit{Fit} 87.31 vs.\ 86.58, \textit{Perturb-mean} 85.88 vs.\ 84.14, \textit{OOD-worst} 78.45 vs.\ 76.47), yielding +1.48 in \textit{Avg} and visibly tightening the tri-objective frontier.
% More broadly, under the same clean-only selection criterion, lifting one axis often fails to lift the others, revealing tri-objective tension within a single protocol.
This suggests that coverage expansion and stability control remain compatible with clean generalization, consistent with \S\ref{sec:H2_main_th}.

\begin{table}[t]
  \caption{Tri-objective ablation on node benchmarks.
 We report \textit{Fit}, \textit{OOD-worst}, \textit{Perturb-mean}, and \textit{Avg}.
}

       \vspace{0in}
  \label{tab:tri_ablation}
  \centering
  % \setlength{\tabcolsep}{3pt}
  % \renewcommand{\arraystretch}{1.10}
  % \scriptsize
  \resizebox{\columnwidth}{!}{%
    \begin{tabular}{llcccc}
      \toprule
      Dataset & Metric       & Full                                              & w/o MoE                                           & w/o VQ                                            & w/o Lip                                           \\
      \midrule
      \multirow{3}{*}{Cora}
              & Fit          & \cellcolor{secondcolor}\underline{87.31}$\pm$1.56 & 86.73$\pm$1.32                                    & 86.57$\pm$1.44                                    & \cellcolor{bestcolor}\textbf{87.78}$\pm$1.39      \\
              & OOD-worst    & \cellcolor{bestcolor}\textbf{78.45}$\pm$0.77      & 77.76$\pm$1.66                                    & 77.01$\pm$1.34                                    & \cellcolor{secondcolor}\underline{78.25}$\pm$1.10 \\
              & Perturb-mean & \cellcolor{bestcolor}\textbf{85.88}$\pm$1.38      & 85.04$\pm$1.47                                    & \cellcolor{secondcolor}\underline{85.37}$\pm$1.46 & 85.33$\pm$1.22                                    \\
      \midrule
      \multirow{3}{*}{PubMed}
              & Fit          & \cellcolor{bestcolor}\textbf{90.14}$\pm$0.62      & 89.07$\pm$0.55                                    & 89.26$\pm$0.35                                    & \cellcolor{secondcolor}\underline{89.30}$\pm$0.58 \\
              & OOD-worst    & \cellcolor{bestcolor}\textbf{84.77}$\pm$0.50      & \cellcolor{secondcolor}\underline{84.14}$\pm$0.35 & 84.09$\pm$0.47                                    & 83.58$\pm$0.49                                    \\
              & Perturb-mean & \cellcolor{bestcolor}\textbf{86.26}$\pm$0.66      & \cellcolor{secondcolor}\underline{85.68}$\pm$0.71 & 85.12$\pm$0.96                                    & 85.37$\pm$0.60                                    \\
      \midrule
      \multirow{3}{*}{WikiCS}
              & Fit          & \cellcolor{bestcolor}\textbf{85.57}$\pm$0.76      & 85.30$\pm$0.56                                    & 84.96$\pm$0.57                                    & \cellcolor{secondcolor}\underline{85.40}$\pm$0.55 \\
              & OOD-worst    & \cellcolor{bestcolor}\textbf{79.02}$\pm$0.67      & 77.43$\pm$0.54                                    & \cellcolor{secondcolor}\underline{78.30}$\pm$0.53 & 77.99$\pm$0.50                                    \\
              & Perturb-mean & \cellcolor{bestcolor}\textbf{83.90}$\pm$0.69      & 82.03$\pm$0.99                                    & \cellcolor{secondcolor}\underline{83.16}$\pm$0.59 & 82.59$\pm$0.68                                    \\
      \midrule
      Avg.    &              & \cellcolor{bestcolor}\textbf{84.59}               & 83.69                                             & 83.76                                             & \cellcolor{secondcolor}\underline{83.95}          \\
      \bottomrule
    \end{tabular}%
  }
  \vspace{-0.2in}
\end{table}

\vspace{3pt}
\noindent\textbf{Ablation Study.}
Table~\ref{tab:tri_ablation} shows the strongest complementarity in the tri-objective setting: the full model achieves the best \textit{Avg} (84.59), and no ablation improves \textit{OOD-worst} and \textit{Perturb-mean} simultaneously.
Removing MoE causes the most consistent robustness drop (e.g., WikiCS \textit{OOD-worst} 79.02$\rightarrow$77.43), supporting MoE as the main driver of mechanism coverage.
VQ stabilizes robustness with similar \textit{Fit} but lower \textit{OOD-worst}/\textit{Perturb-mean} (e.g., Cora \textit{OOD-worst} 78.45$\rightarrow$77.01), while Lipschitz reveals a fit--robustness tradeoff (without Lip, Cora \textit{Fit} rises $87.31\!\rightarrow\!87.78$ but \textit{Avg} falls $84.59\!\rightarrow\!83.95$).
Overall, MoE expands coverage under shift, VQ stabilizes the encoder-head interface under drift, and Lipschitz control limits amplification under perturbations; a Vanilla MoE (w/o VQ \& Lip) stress test under stronger feature masking is reported in Appendix~\ref{app:additional_ablation}.

% Routing+switch figure moved to appendix (see app:routing_switch_fig).

\vspace{-0.1in}
\section{Conclusion}

We show that the test-time computation mechanism - not the training objective alone - governs the tri-objective balance in frozen graph deployment.
Our analysis explains why static inference hits a structural floor and how instance-conditional computation adds leverage but introduces routing failure modes.
STEM-GNN addresses these levers via MoE routing, VQ tokenization, and Lipschitz control; clean and tri-objective ablations confirm component complementarity.
Under clean-only selection with OOD and perturbation evaluation, STEM-GNN achieves a better balance through MoE coverage and VQ+Lip stability control in practice.
This shows that expanding mechanisms and constraining sensitivity can coexist in a single frozen model.
However, fixed expert count and a frozen codebook may limit adaptability to distant domains.
Adaptive expert allocation, online codebook refinement, and temporal drift are promising future research directions.

\begin{acks}
This work was partially supported by the NSF under grants IIS-2528540, IIS-2334193, IIS-2340346, CNS-2426514, and CMMI-2146076. This work also used computational resources provided through NSF ACCESS grant CIS260048. Any opinions, findings, conclusions, or recommendations expressed in this material are those of the authors and do not necessarily reflect the views of the sponsors.
\end{acks}

% References
\bibliographystyle{ACM-Reference-Format}
\balance
\bibliography{references}

% \clearpage
% \appendix
% \onecolumn

\appendix

% =========================================================
% Appendix: H1 details (complete, self-contained)
% =========================================================
\section{Proof of the Static-Inference Witness Tension}
\label{app:H1}

\paragraph{Setup.}
A predictor $f_\theta\in\mathcal{H}_1$ applies a single computation rule under frozen inference (no instance-wise routing). Fix graph operators $T_{\mathrm{L}}(G), T_{\mathrm{H}}(G)$ that decompose node features into complementary components $X_{\mathrm{L}}:=T_{\mathrm{L}}(G)X$ and $X_{\mathrm{H}}:=T_{\mathrm{H}}(G)X$ with $T_{\mathrm{L}}(G)+T_{\mathrm{H}}(G)=I$ (e.g., low-/high-pass filters). The witness subfamily $\mathcal{H}_1^{\mathrm{wit}}\subseteq\mathcal{H}_1$ consists of static predictors
\begin{equation}
f_\theta(G,X)\ :=\ h_\theta\!\big(X_{\mathrm{L}} + \eta_\theta X_{\mathrm{H}}\big),\qquad \eta_\theta\ge 0,
\label{eq:H1_proxy_mix_app}
\end{equation}
where $\eta_\theta$ is fixed after training and $h_\theta$ is the task head. Throughout this appendix we specialize the perturbation set to feature perturbations only (structure budget $\rho_s=0$), with admissible set $B(G,X):=\{X':(G,X')\in B((G,X))\}$.

\begin{assumption}[High-component necessity and perturbations]
\label{ass:H1_conflict}
Fix $\alpha>0$. There exist a reference distribution $D_0$ over $(G,X)$, a witness test environment $e_1\in E_{\mathrm{test}}$ with distribution $D_{e_1}$ over $((G,X),y)$, and constants $\rho>0,L_h>0$ such that:
(i)~there exists a monotone decreasing $\psi_{\mathrm{fit}}:[0,\infty)\to\mathbb{R}_+$ with $\mathcal{R}_{\mathrm{fit}}(\theta)\ge\psi_{\mathrm{fit}}(\eta_\theta)$ and $\psi_{\mathrm{fit}}(0)>\alpha$;
(ii)~there exists a monotone decreasing $\psi_{e_1}:[0,\infty)\to\mathbb{R}_+$ with $\mathbb{E}_{D_{e_1}}[\ell(f_\theta(G,X),y)]\ge\psi_{e_1}(\eta_\theta)$ and $\psi_{e_1}(0)>0$;
(iii)~admissible perturbations act on the high component, i.e., for $(G,X)\sim D_0$ and any $X'\in B(G,X)$, $X'_{\mathrm{L}}=X_{\mathrm{L}}$ and $\|X'_{\mathrm{H}}-X_{\mathrm{H}}\|\le\rho$;
(iv)~the head $h_\theta$ is $L_h$-Lipschitz: $d_{\mathrm{out}}(h_\theta(u),h_\theta(v))\le L_h\|u-v\|$ for all $u,v$.
\end{assumption}

\begin{proof}[Proof of Theorem~\ref{thm:H1_lower_main_th}]
Recall $\eta_{\max}(\varepsilon):=\varepsilon/(L_h\rho)$ and $\eta_{\min}(\alpha):=\inf\{\eta\ge 0:\psi_{\mathrm{fit}}(\eta)\le\alpha\}$.

\smallskip\noindent\textbf{Step 1 (the constraint caps $\eta_\theta$).}
For $(G,X)$ and $X'\in B(G,X)$, Assumption~\ref{ass:H1_conflict}(iii) gives $X'_{\mathrm{L}}=X_{\mathrm{L}}$ and $\|X'_{\mathrm{H}}-X_{\mathrm{H}}\|\le\rho$, so the input to $h_\theta$ differs by at most $\eta_\theta\rho$. Assumption~\ref{ass:H1_conflict}(iv) then yields
\[
\sup_{X'\in B(G,X)} d_{\mathrm{out}}\big(f_\theta(G,X),f_\theta(G,X')\big)\ \le\ L_h\,\eta_\theta\,\rho,
\]
and hence $\mathcal{R}_{\mathrm{stab}}(\theta)\le L_h\,\eta_\theta\,\rho$ in expectation. The Lipschitz stability constraint $L_h\,\eta_\theta\,\rho\le\varepsilon$ rearranges to $\eta_\theta\le\eta_{\max}(\varepsilon)$.

\smallskip\noindent\textbf{Step 2 (fit forces a minimum $\eta_\theta$).}
Assumption~\ref{ass:H1_conflict}(i) gives $\mathcal{R}_{\mathrm{fit}}(\theta)\ge\psi_{\mathrm{fit}}(\eta_\theta)$. Combined with $\mathcal{R}_{\mathrm{fit}}(\theta)\le\alpha$, the definition of $\eta_{\min}(\alpha)$ forces $\eta_\theta\ge\eta_{\min}(\alpha)$.

\smallskip\noindent\textbf{Step 3 (combine).}
If $\eta_{\min}(\alpha)>\eta_{\max}(\varepsilon)$, no $\theta\in\mathcal{H}_1^{\mathrm{wit}}$ satisfies both constraints. Otherwise $\eta_\theta\in[\eta_{\min}(\alpha),\eta_{\max}(\varepsilon)]$, and by Assumption~\ref{ass:H1_conflict}(ii) with $\psi_{e_1}$ monotone decreasing,
\[
\mathcal{R}_{\mathrm{ood}}(\theta)\ \ge\ \mathbb{E}_{D_{e_1}}\!\big[\ell(f_\theta(G,X),y)\big]\ \ge\ \psi_{e_1}(\eta_\theta)\ \ge\ \psi_{e_1}\!\big(\eta_{\max}(\varepsilon)\big),
\]
which is $>0$ by Assumption~\ref{ass:H1_conflict}(ii).
\end{proof}

\section{Proofs for ICC Coverage and Stability Levers}
\label{app:H2_details}

This appendix proves the two main-text levers, Eq.~\eqref{eq:H2_ood_bound_main_th} (coverage--selection) and Eq.~\eqref{eq:H2_stab_decomp_main_th} (stability decomposition).

\subsection{Coverage-selection bound}
\label{app:H2_coverage}

Consider a finite family of mechanisms $\mathcal{F}=\{f^{(k)}\}_{k\in\mathcal{K}}$ and a fixed deterministic discretization $\kappa:\mathcal{R}\to\mathcal{K}$, inducing a discrete selector $\tilde r_\theta(z):=\kappa(r_\theta(z))$ and a routed predictor $f_\theta(z):=f^{(\tilde r_\theta(z))}(z)$. For each $e\in E_{\mathrm{test}}$, let
\[
\begin{aligned}
k^\star(e)
&\in \arg\min_{k\in\mathcal{K}}
\mathbb{E}_{(z,y)\sim D_e}[\ell(f^{(k)}(z),y)],\\
\delta_{\mathrm{sel}}(e)
&:= \mathbb{P}_{z\sim D_e}[\tilde r_\theta(z)\neq k^\star(e)],
\end{aligned}
\]
\[
\beta_{\mathrm{cov}}:=\sup_{e\in E_{\mathrm{test}}}\min_{k\in\mathcal{K}}\mathbb{E}_{(z,y)\sim D_e}[\ell(f^{(k)}(z),y)].
\]
We assume the evaluation loss is bounded: $0\le\ell\le L_{\max}$ on the support of $\bigcup_{e\in E_{\mathrm{test}}} D_e$ (training may still use an unbounded surrogate).

\begin{proof}[Proof of Eq.~\eqref{eq:H2_ood_bound_main_th}]
Pointwise, for any $(z,y)$ and $e\in E_{\mathrm{test}}$,
\[
\ell(f_\theta(z),y)\ \le\ \ell\big(f^{(k^\star(e))}(z),y\big)+L_{\max}\,\mathbf{1}\{\tilde r_\theta(z)\neq k^\star(e)\},
\]
since the indicator switches on iff the selected mechanism differs from $k^\star(e)$, and the loss is bounded by $L_{\max}$. Taking expectation under $D_e$ and using the definitions of $\beta_{\mathrm{cov}}$ and $\delta_{\mathrm{sel}}(e)$,
\[
\mathbb{E}_{(z,y)\sim D_e}[\ell(f_\theta(z),y)]\ \le\ \beta_{\mathrm{cov}}+L_{\max}\,\delta_{\mathrm{sel}}(e).
\]
Taking $\sup_{e\in E_{\mathrm{test}}}$ yields Eq.~\eqref{eq:H2_ood_bound_main_th}.
\end{proof}

\subsection{Stability decomposition}
\label{app:H2_stability}

Equip $\mathcal{R}$ with a distance $\mathrm{dist}_{\mathcal{R}}$. Define the routing-fixed sensitivity and routing drift
\[
\mathcal{R}_{\mathrm{base}}(\theta):=\mathbb{E}_{z\sim D_0}\!\Big[\sup_{z'\in B(z)} d_{\mathrm{out}}\big(F(r_\theta(z),z),F(r_\theta(z),z')\big)\Big],
\]
\[
\mathcal{R}_{\mathrm{route}}(\theta):=\mathbb{E}_{z\sim D_0}\!\Big[\sup_{z'\in B(z)}\mathrm{dist}_{\mathcal{R}}\big(r_\theta(z),r_\theta(z')\big)\Big].
\]
Let the pointwise routing envelope and the uniform envelope (assumed finite) be
\[
L_F(z)
:=
\sup_{\substack{r,r'\in\mathcal{R}\\ \mathrm{dist}_{\mathcal{R}}(r,r')>0}}
\frac{d_{\mathrm{out}}(F(r,z),F(r',z))}{\mathrm{dist}_{\mathcal{R}}(r,r')},
\]
\[
L_F^B(D_0)
:=
\sup_{z\in\mathrm{supp}(D_0)}\ \sup_{z'\in B(z)} L_F(z').
\]

\begin{proof}[Proof of Eq.~\eqref{eq:H2_stab_decomp_main_th}]
Fix $z\sim D_0$ and any $z'\in B(z)$. Write $r=r_\theta(z),r'=r_\theta(z')$. By triangle inequality,
\[
d_{\mathrm{out}}\big(F(r,z),F(r',z')\big)\ \le\ d_{\mathrm{out}}\big(F(r,z),F(r,z')\big)\ +\ d_{\mathrm{out}}\big(F(r,z'),F(r',z')\big).
\]
The pointwise envelope gives $d_{\mathrm{out}}(F(r,z'),F(r',z'))\le L_F(z')\,\mathrm{dist}_{\mathcal{R}}(r,r')$. Taking $\sup_{z'\in B(z)}$ then $\mathbb{E}_{z\sim D_0}$, the first term yields $\mathcal{R}_{\mathrm{base}}(\theta)$. Since $L_F^B(D_0)$ is a uniform constant over $\mathrm{supp}(D_0)$, it factors out of the expectation, and the second term is bounded by $L_F^B(D_0)\,\mathcal{R}_{\mathrm{route}}(\theta)$.
\end{proof}

\section{Proofs for VQ and Lipschitz Control}
\label{sec:appendix_method_details}

\subsection{Proof of Lemma~\ref{lem:vq_margin_invariance}}
\label{sec:appendix_vq_proof}

\begin{proof}
Fix an input $z$ and a node $v\in V$, and let $j^\star := q_\theta(v;z) = \arg\min_{1\le j\le M}\|u_v(z)-c_j\|_2$. Recall the margin
\[
m_v(z) := \min_{j\neq j^\star}\|u_v(z)-c_j\|_2 - \|u_v(z)-c_{j^\star}\|_2.
\]
Fix any $z'\in B(z)$ and set $\delta := \|u_v(z')-u_v(z)\|_2$; by assumption, $\delta < \tfrac{1}{2}m_v(z)$. Triangle inequality gives, for any competitor $j\neq j^\star$,
\[
\begin{aligned}
\|u_v(z')-c_{j^\star}\|_2 &\le \|u_v(z)-c_{j^\star}\|_2 + \delta,\\
\|u_v(z')-c_j\|_2 &\ge \|u_v(z)-c_j\|_2 - \delta.
\end{aligned}
\]
Hence
\[
\|u_v(z')-c_j\|_2 - \|u_v(z')-c_{j^\star}\|_2 \ge m_v(z) - 2\delta > 0,
\]
so the nearest-code assignment is unchanged: $q_\theta(v;z') = j^\star = q_\theta(v;z)$. Since this holds for every $v\in V$ and every $z'\in B(z)$, $Q(u_\theta(z')) = Q(u_\theta(z))$ for all $z'\in B(z)$.
\end{proof}

\subsection{Proof of Proposition~\ref{prop:vq_lip_combo}}
\label{sec:appendix_vq_lip_bound}

Let $\mathcal{C}=\{c_j\}_{j=1}^{M}$ be the VQ codebook (the implementation uses grouped $H$-head VQ with $H{=}4$; $\mathcal{C}$ then denotes the concatenated per-head token space, and the invariance condition in Lemma~\ref{lem:vq_margin_invariance} is applied per head) and define its diameter
\[
\mathrm{diam}(\mathcal{C}) := \max_{i,j}\|c_i-c_j\|_2 .
\]
Let the token embedding for node $v$ be $\tilde u_v(z)=Q(u_v(z))\in\mathcal{C}$.

\paragraph{Node-level tasks (no pooling).}
For any $z'\in B(z)$ and any node $v$, both $\tilde u_v(z')$ and $\tilde u_v(z)$ lie in $\mathcal{C}$, hence
\[
\|\tilde u_v(z')-\tilde u_v(z)\|_2 \le \mathrm{diam}(\mathcal{C}).
\]
Since the head is $L$-Lipschitz under $\|\cdot\|_2$,
\[
\|f_t(\tilde u_v(z'))-f_t(\tilde u_v(z))\|_2
\le
L\|\tilde u_v(z')-\tilde u_v(z)\|_2
\le
L\,\mathrm{diam}(\mathcal{C}).
\]

\paragraph{Graph-level tasks (with pooling).}
Let $A(z)=\{\tilde u_v(z)\}_{v\in V}$ denote the token set.
Assume the pooling operator is $L_{\mathrm{pool}}$-Lipschitz with respect to the nodewise norm
\[
\|A\|_{V,2} := \Big(\frac{1}{|V|}\sum_{v\in V}\|a_v\|_2^2\Big)^{1/2},
\]
namely,
\[
\|\mathrm{Pool}(A)-\mathrm{Pool}(B)\|_2 \le L_{\mathrm{pool}}\|A-B\|_{V,2}.
\]
The per-node bound $\|\tilde u_v(z')-\tilde u_v(z)\|_2 \le \mathrm{diam}(\mathcal{C})$ implies
\[
\|A(z')-A(z)\|_{V,2} \le \mathrm{diam}(\mathcal{C}).
\]
Let $r(z)=\mathrm{Pool}(A(z))$. Then
\[
\|r(z')-r(z)\|_2 \le L_{\mathrm{pool}}\,\mathrm{diam}(\mathcal{C}),
\]
and by the $L$-Lipschitzness of $f_t$,
\[
\|f_t(r(z'))-f_t(r(z))\|_2
\le
L\|r(z')-r(z)\|_2
\le
L\,L_{\mathrm{pool}}\,\mathrm{diam}(\mathcal{C}).
\]
This matches Eq.~\eqref{eq:vq_lip_combo_body} ($L_{\mathrm{pool}}=1$ recovers the node-level bound).
\qed

\section{Experimental Details}
\label{sec:appendix_experiments}

\subsection{Routing Specialization and VQ Stability}
\label{app:routing_switch_fig}

Panel~(a) splits WikiCS test nodes by local feature homophily ($h\leq 0.33$ vs.\ $h>0.67$): Expert-0 drops $62\%{\to}46\%$ and Expert-2 rises $25\%{\to}43\%$, an $18$-point top-1 swing that rules out router collapse and supports the coverage mechanism behind $\beta_{\mathrm{cov}}$ in Eq.~\ref{eq:H2_ood_bound_main_th}. Panel~(b) reports the per-head mean codebook margin against the switch rate at $p{=}0.1$ as a descriptive summary; across the three datasets larger mean margin coincides with lower switch rate (WikiCS $0.62$ / $6.2\%$; PubMed $0.29$ / $15.4\%$). Consistent with the per-head mechanism of Lemma~\ref{lem:vq_margin_invariance}, the sufficient invariance bound is governed by the minimum-margin head; the mean is shown only as a coarse aggregate and is not the quantity the lemma controls. The residual amplification of these switches is bounded by the Lipschitz penalty (Eq.~\ref{eq:lip_fro_penalty}) on the prediction head.

\begin{figure}[h]
  \centering
  \includegraphics[width=\linewidth]{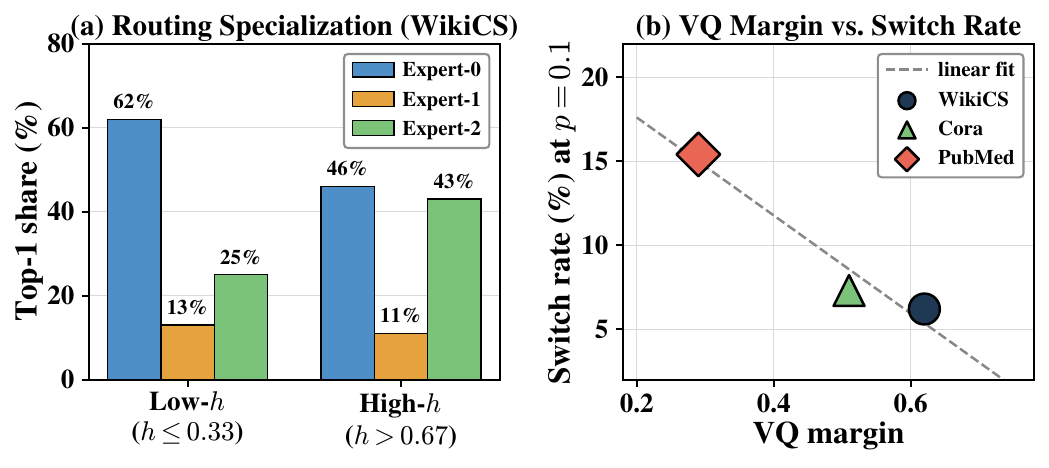}
  \vspace{-0.2in}
  \caption{Mechanism diagnostics. (a) Top-1 expert distribution on WikiCS, split by local feature homophily. (b) Per-dataset per-head mean codebook margin vs.\ token switch rate at $p{=}0.1$.}
  \vspace{-0.15in}
  \label{fig:routing_switch}
\end{figure}

\subsection{Additional Ablation Studies}
\label{app:additional_ablation}

\paragraph{Joint VQ and Lipschitz ablation under perturbation.}
We jointly remove both the VQ codebook and the Lipschitz penalty from STEM-GNN, retaining only the MoE encoder (\emph{Vanilla MoE}, w/o VQ \& Lip), and evaluate on Cora and PubMed under feature masking at $\alpha\in\{0.5,0.6\}$ (Table~\ref{tab:vanilla_moe_ablation}). The gap widens sharply with masking severity ($-2.29$ / $-2.67$ at $\alpha{=}0.6$ on Cora / PubMed), supporting that VQ and Lipschitz jointly bound the amplification of perturbation-induced representation and token-switch variation.

\begin{table}[t]
\centering
\caption{Joint ablation of VQ and Lipschitz under feature masking. Numbers are node-classification accuracy (\%); $\Delta$ is Vanilla MoE minus STEM-GNN.}
\label{tab:vanilla_moe_ablation}
\small
\begin{tabular}{lccc}
\toprule
Setting & STEM-GNN & Vanilla MoE (w/o VQ \& Lip) & $\Delta$ \\
\midrule
Cora, $\alpha{=}0.5$    & 75.98 & 75.41 & $-0.57$ \\
Cora, $\alpha{=}0.6$    & 73.77 & 71.48 & $\mathbf{-2.29}$ \\
PubMed, $\alpha{=}0.5$  & 74.50 & 74.31 & $-0.19$ \\
PubMed, $\alpha{=}0.6$  & 74.32 & 71.65 & $\mathbf{-2.67}$ \\
\bottomrule
\end{tabular}
\end{table}

\subsection{Additional Diagnostics}

\paragraph{Compute and parameter footprint.}
On Cora (NVIDIA L40), STEM-GNN trains at $5.24{\pm}0.04$~ms/epoch and infers at $2.72{\pm}0.09$~ms/pass with $279$~MB peak memory. The model has $9.93$~M parameters versus $7.57$~M for GFT; the $K{=}3$ MoE router is a single linear layer ($2.3$~K params), and the Lipschitz penalty is loss-only ($0$ params). Soft routing (Eq.~\ref{eq:moe_update_soft}) evaluates all experts in one forward pass, avoiding top-$k$ dispatch overheads of sparse MoE.

\paragraph{Statistical reliability and sensitivity.}
For the tri-objective evaluation, $50/60$ comparisons against five baselines show large effects ($|d|>0.8$), with $47/50$ of these having $95\%$ confidence intervals excluding zero. Sensitivity sweeps show only modest variation: changing the routing temperature $\tau\in[0.3,2.0]$ on WikiCS changes OOD accuracy by $<1.3$~pp, and changing the Lipschitz coefficient $\lambda_{\mathrm{lip}}\in[0,10^{-2}]$ on Cora changes it by $<1.0$~pp.

\paragraph{Per-dataset perturbation results.}
Figure~\ref{fig:noise_appendix} reports the remaining perturbation settings (PubMed/WikiCS feature masking, Cora/ WikiCS edge deletion); STEM-GNN stays the most stable across all panels and severities ($p\in\{0.4,0.5,0.6\}$).

\FloatBarrier

\begin{figure}[ht]
  \centering
  \includegraphics[width=\linewidth]{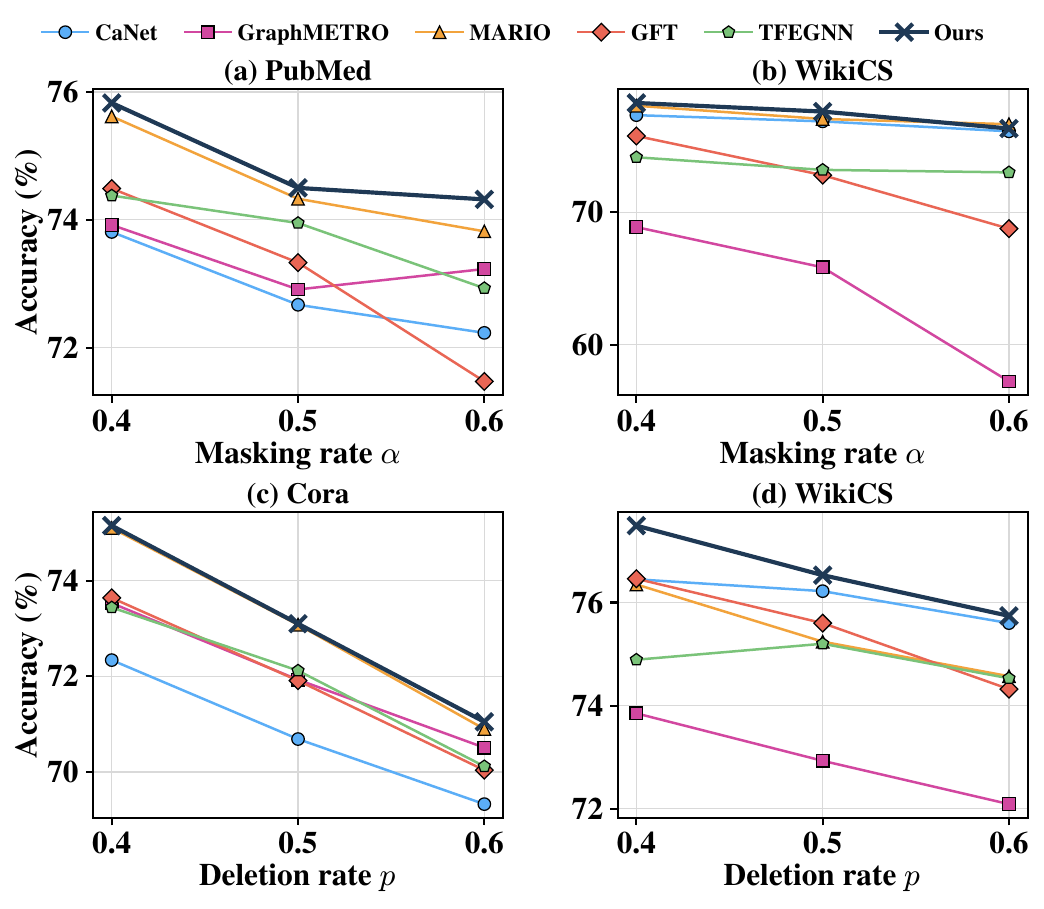}
  \vspace{-0.2in}
  \caption{Stability under inference-time perturbation (accuracy, \%) across feature masking and edge deletion.}
  \vspace{-0.1in}
  \label{fig:noise_appendix}
\end{figure}

\subsection{Hyperparameters}

\paragraph{Pre-training.}
The encoder is a two-layer GraphSAGE backbone (hidden $768$, ReLU, batch normalization, dropout $0.15$) with a vector-quantized token interface (codebook $M{=}128$, dim $768$, $4$ grouped heads, commitment weight $10$) and a last-layer MoE of $K{=}3$ experts with Gumbel-Softmax routing ($\tau{=}1.0$). Pre-training minimizes three objectives on the tokenized representations---node-feature reconstruction, link reconstruction over $10\%$ held-out edges, and semantic alignment---together with the VQ commitment loss, using feature masking and edge dropping (rate $0.2$). We optimize with AdamW (lr $1\mathrm{e}{-4}$, weight decay $1\mathrm{e}{-5}$, cosine schedule) for $25$ epochs at batch size $1024$.

\paragraph{Fine-tuning.}
We fine-tune with AdamW for frozen deployment: the VQ codebook is frozen ($\texttt{freeze\_vq=1}$), the last-layer MoE encoder ($K{=}3$) and a linear, Lipschitz-regularized head are optimized, and checkpoints are selected by validation accuracy. Batch normalization is used for link/graph tasks and disabled for node tasks.

\end{document}